\pdfoutput=1
\documentclass[11pt]{article}

\usepackage[preprint]{acl}
\usepackage{microtype}

\usepackage[T1]{fontenc}
\usepackage[utf8]{inputenc}
\usepackage{tgtermes}

\usepackage{amsmath,amssymb,amsthm,mathtools}
\usepackage{booktabs,tabularx,array,multirow}
\usepackage{graphicx}
\usepackage{xcolor}
\usepackage{enumitem}
\graphicspath{{fig/}}

\theoremstyle{plain}
\newtheorem{theorem}{Theorem}
\newtheorem{proposition}{Proposition}
\newtheorem{corollary}{Corollary}
\newtheorem{lemma}{Lemma}
\theoremstyle{definition}

\theoremstyle{remark}

\newcommand{\KL}{\mathrm{KL}}
\newcommand{\TV}{\mathrm{TV}}
\newcommand{\E}{\mathbb{E}}

\newcommand{\Dstep}{D^{2}_{S}}
\newcommand{\Dtraj}{D^{2}_{\tau}}

\newcommand{\base}{\mathrm{base}}
\newcommand{\Tmax}{T_{\max}}
\DeclarePairedDelimiter{\abs}{\lvert}{\rvert}
\DeclarePairedDelimiter{\norm}{\lVert}{\rVert}

\title{Agent Behavior Maps: Locating Agents by What They Would Do,\\ Without Running Them}

\author{Shota Fujikawa\thanks{\ Equal contribution.}\thanks{\ Work done during an internship at NEC Corporation.} \\
  The University of Tokyo \\ \texttt{fjkwsht@is.s.u-tokyo.ac.jp} \And
  Takuya Tamura$^{*}$ \\
  NEC Corporation \\ \texttt{tamura-takuya@nec.com} \AND
  Kosuke Akimoto \\
  NEC Corporation \\ \texttt{kosuke\_a@nec.com} \And
  Daichi Haraguchi \\
  NEC Corporation \\ \texttt{daichi-haraguchi@nec.com} \And
  Masafumi Oyamada \\
  NEC Corporation \\ \texttt{oyamada@nec.com}}

\begin{document}
\maketitle

\begin{abstract}
LLM agents with similar success rates can follow different action sequences. An Agent Behavior Map represents LLM--harness combinations using centered action log-likelihoods on shared reference trajectories, without requiring new environment rollouts from the target agent. We derive error bounds relating squared coordinate distances to trajectory KL under explicit assumptions, and across eight cross-scored models the distances track the directly measured trajectory KL in rank order. Across 100 LLMs and four environments, ALFWorld-based coordinates support performance prediction and identify behavioral similarities. In ALFWorld, these similarities persist after controlling for success rates and model size, including across model families. As an extension, harness-component displacements approximately compose, with accuracy varying across components and environments.
\end{abstract}

\section{Introduction}
\label{sec:intro}

\begin{figure*}[t]
  \centering
  \includegraphics[width=\textwidth]{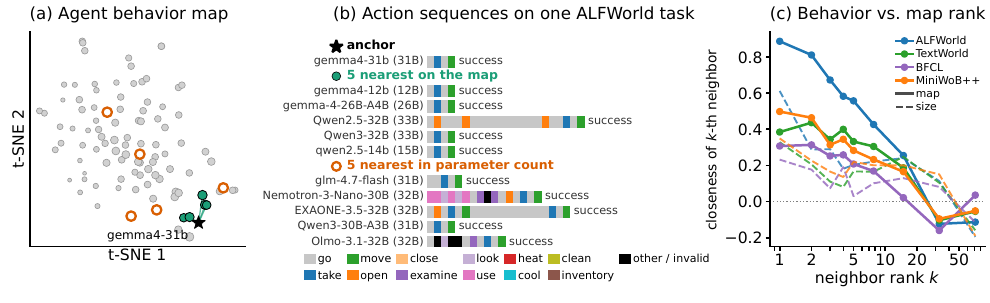}
  \caption{\textbf{Agents close on the map act alike when run.} (a) 100 LLMs with the BASE harness, scored on 500 ALFWorld reference trajectories (t-SNE layout; neighbors computed in the original space). $\star$: gemma4-31b; green: its five nearest models on the map; orange rings: the five closest in parameter count. (b) Action sequences of these 11 models on one ALFWorld task. All succeed, but four of the five map neighbors share the anchor's short route, whereas three of the five size-matched models take longer ones. (c) Effect size of the behavioral distance to the \mbox{$k$-th} nearest model on the map (solid) or in parameter count (dashed), averaged over four aspects and 100 models (\S\ref{sec:behaviour}); positive means more similar than the average pair.}
  \label{fig:map}
\end{figure*}

Comparing LLM agents is a matter not only of how often they succeed but also of how they act. Agents with similar success rates can differ in the tools they choose, the order of their actions, and how often they repeat an operation. Such behavior depends on the combination of the underlying LLM and the harness that manages history, planning, and reasoning~\citep{zhang2026stopcomparing,yao2026harnessbench,vats2026scaffold,fan2026harnessdesign}. Running every combination in the environment reveals both performance and behavior, but each new agent to compare requires collecting new trajectories. This paper asks whether a representation that captures similarity of run-time behavior can be built without new environment runs of the target agent.

When agents are compared by their log-likelihoods on shared reference trajectories, the difficulty is that the distribution of the histories on which the likelihoods are conditioned itself depends on the agent's policy. The histories contained in the reference trajectories need not coincide with those that an agent reaches through its own actions. Whether similarity of action probabilities measured on shared histories corresponds to similarity of trajectory distributions and behavior at run time is therefore not obvious. How much of these on-policy properties can be captured by off-policy observations on fixed reference trajectories is the central question of this work.

Language model maps, which vectorize log-likelihoods on shared texts, represent the relations between models in a common coordinate space~\citep{oyama2025mapping}. Their extension to prompt--response distributions has examined performance prediction from conditional log-likelihoods and the approximate additivity of the coordinate displacements caused by prompt changes~\citep{takase2026maps}. This paper applies the framework to agents that interact sequentially with an environment and answers the question above with theory and experiments.

An Agent Behavior Map treats each agent as a pair of an LLM and a harness. Each agent rebuilds its input with its own harness from the shared external history in the reference trajectories, namely the task, the observations, and the past actions. The recorded actions are teacher-forced under this input, and the resulting log-likelihood vector is centered to give the agent's coordinates. Because the reference trajectories are reused, the coordinates of a new agent can be computed without collecting any additional environment trajectories from it. Figure~\ref{fig:map} shows an example in which agents that are close in these coordinates take similar action sequences when run. We validate the representation from two angles.
\begin{itemize}[leftmargin=*,itemsep=1pt,topsep=2pt]
  \item \textbf{Theoretical and empirical relation to the trajectory distribution} (\S\ref{sec:method}, \S\ref{sec:klfit}). Under assumptions on the locality of the two policies and on the mismatch with the reference distribution, we derive error bounds that relate squared coordinate distances to the KL divergence between trajectory distributions. The error decomposes into local approximation, reference mismatch, centering, and finite samples, and the decomposition shows how the step-level and trajectory-level representations differ in their dependence on trajectory length. In a cross-scoring experiment with eight models, we compare the coordinate distances with the directly measured trajectory KL, the target itself, and find a high rank correlation.
  \item \textbf{Validation on performance and behavior} (\S\ref{sec:predicting}, \S\ref{sec:behaviour}). We apply a common harness to 100 LLMs and evaluate coordinates built from ALFWorld reference trajectories alone on ALFWorld, TextWorld, BFCL, and MiniWoB++. Besides predicting the success rates of unevaluated models from the evaluation results of the others, we ask whether neighbors on the map generate similar action sequences. In ALFWorld in particular, the similarity of action sequences remains when compared with neighbors chosen by success rate or by size, and when the candidates are restricted to models of other families. These results support the view that the coordinates carry behavioral information that neither performance level nor model family explains. The accuracy of performance prediction and the strength of behavioral similarity among neighbors, however, depend on the environment and the measure.
\end{itemize}
Finally, we represent the addition of a harness component as a displacement on the map and ask whether such displacements add up within an LLM (\S\ref{sec:harness}). Comparison with measured coordinates shows approximate additivity, with an accuracy that depends on the component and the environment.

\section{Agent Behavior Maps}
\label{sec:method}
\subsection{Problem: comparing agents by what they do when run}
\label{sec:targets}
On task $x$, an agent interacts with the environment and produces a trajectory $\tau=(x,o_0,a_0,o_1,\dots,a_{T-1},o_T)$ of observations $o_t$ and actions $a_t$. Its length $T=T(\tau)$ varies and is capped by a known limit $\Tmax$. The \textbf{external history} before the action at time $t$ is $h_t=(x,o_0,a_0,\dots,o_t)$; it excludes the generator's internal state, such as its system prompt, reasoning, plans, and summaries. An agent $i=(\text{LLM},\text{harness})$ rebuilds its input from $h_t$ by its own rule $C_i$ (history window, retrieval, summarization) and generates text under the action grammar; $\pi_i(a\mid h)$ denotes the resulting probability of action $a$ (Appendix~\ref{app:probability} covers the grammar and tokenization). Given a task distribution $w$, an initial observation $p_0$, and environment dynamics $E_t=E(o_{t+1}\mid h_t,a_t)$ shared by all agents, agent $i$ induces the trajectory distribution
\begin{equation}
P_i(\tau)=w(x)\,p_0(o_0\mid x)\prod_{t<T}\pi_i(a_t\mid h_t)\,E_t.
\label{eq:traj}
\end{equation}
The environment terms cancel in the log-likelihood ratio of two agents. We compare agents by the trajectory KL $K_{ij}:=\KL(P_i\Vert P_j)$ and its symmetrization $K_{\mathrm{sym}}:=(K_{ij}+K_{ji})/2$. $P_i$ is the distribution of trajectories when agent $i$ is actually run, so the properties we care about, such as success rates and action frequencies, are functionals of $P_i$: on-policy quantities.

Writing $\E_i=\E_{P_i}$, the chain rule of KL~\citep{cover2006elements,ke2019fdiv} gives
\begin{equation}
K_{ij}=\E_i\textstyle\sum_{t<T}\KL\bigl(\pi_i(\cdot\mid h_t)\Vert\pi_j(\cdot\mid h_t)\bigr),
\label{eq:chain}
\end{equation}
an expectation over the histories that agent $i$ \emph{visits when it runs}, so measuring it directly takes $n$ rollouts and $O(n^2)$ cross-scorings for $n$ agents. We ask how much of this quantity can be read, \emph{without running the agents}, from distances between coordinates obtained by forward passes alone.

\subsection{Scoring agents on shared reference trajectories}
\label{sec:reference}
We prepare $N$ \textbf{reference trajectories} $\tau_1,\dots,\tau_N$ once, from an unknown reference distribution $Q$. The distances use neither $Q(\tau)$ nor the generators or their mixing weights, so the trajectories can come from anywhere. We use four sources: logs of the agents under study (ALFWorld, Appendix~\ref{app:harness}), traces published by others ($\tau$-bench~\citep{yao2025taubench} and SWE-Bench~\citep{jimenez2024swebench} from HAL~\citep{kapoor2025hal}), rollouts of two models (the 100-model experiment, \S\ref{sec:mapping}), and an equal-weight pool of eight scored models (\S\ref{sec:klfit}). Let $T_r=T(\tau_r)$ be the lengths and $M=\sum_rT_r$ the total number of steps. For each agent $i$, teacher forcing~\citep{williams1989teacher} reads $\ell_{i,rt}:=\log\pi_i(a_{rt}\mid h_{rt})$ and $A_{i,r}:=\sum_{t<T_r}\ell_{i,rt}$ without running the environment. Sharing the trajectories has two costs, the mismatch between $Q$ and each agent's on-policy distribution and the finite number of trajectories, which \S\ref{sec:estimator} and \S\ref{sec:finite} address.

\subsection{Step and trajectory embeddings}
\label{sec:vectors}
\label{sec:estimator}
Let $\bar\ell_i$ and $\bar A_i$ be the means of $\ell_{i,rt}$ over all $M$ steps and of $A_{i,r}$ over all $N$ trajectories. Subtracting \emph{each agent's own mean}, we define two centered embeddings
\begin{equation*}
z^{S}_i:=M^{-\frac12}(\ell_{i,rt}-\bar\ell_i)_{r,t},\,
z^{\tau}_i:=N^{-\frac12}(A_{i,r}-\bar A_i)_{r}.
\end{equation*}
A step coordinate asks whether an action is natural in its situation; a trajectory coordinate asks whether a trajectory is natural from its initial state. Only existing steps $t<T_r$ are included. The distances are $\hat D_S=\norm{z^S_i-z^S_j}$ and $\hat D_\tau=\norm{z^\tau_i-z^\tau_j}$. For a pair $(i,j)$, let $d_t=\log\frac{\pi_i(a_t\mid h_t)}{\pi_j(a_t\mid h_t)}$ be the step log-ratio, $X=\sum_{t<T}d_t=\log\frac{P_i(\tau)}{P_j(\tau)}$ its sum, $\mu=\E_QX$ its mean under the reference, and $\bar T=\E_Q[T]$ the mean length. The population distances are then $\Dstep=\frac1{\bar T}\E_Q\sum_{t<T}(d_t-\mu/\bar T)^2$ and $\Dtraj=\mathrm{Var}_Q(X)$. Centering removes each agent's overall confidence, its mean log-likelihood, from the coordinates (Appendices~\ref{app:uncentered} and \ref{app:centering} treat the uncentered variants).

\paragraph{Distances estimate the trajectory KL.}
Fix a pair $(i,j)$. \emph{Assumption~1 (locality):} for every history $h$ that occurs under $P_i$, $P_j$, or $Q$, $\pi_i(\cdot\mid h)$ and $\pi_j(\cdot\mid h)$ have the same support and $\abs{\log\pi_i(a\mid h)/\pi_j(a\mid h)}\le\varepsilon$ on it. \emph{Assumption~2 (reference mismatch):} $\TV(P_i,Q)\le\lambda$ and $\TV(P_j,Q)\le\lambda$. Lengths satisfy $1\le T\le\Tmax$; we assume neither independence across steps nor independence between $T$ and the actions.
\begin{theorem}[Approximation error]
\label{thm:approx}
Under Assumptions~1 and~2, $\abs{K_{ij}-\tfrac{\bar T}{2}\Dstep}\le R_S$ and $\abs{K_{ij}-\tfrac12\Dtraj}\le R_\tau$, where, up to absolute constants,
\begin{align}
R_S&=O\bigl(\varepsilon^3\E_i[T]+\lambda\Tmax\varepsilon^2+\mu^2/\bar T\bigr),\label{eq:RS}\\
R_\tau&=O\bigl(\varepsilon^3\E_i[T^2]+\lambda\Tmax^2\varepsilon^2+\mu^2\bigr).\label{eq:Rt}
\end{align}
Analogous bounds hold for $K_{\mathrm{sym}}$ (Appendix~\ref{app:proofs}).
\end{theorem}
The factor $\Tmax$ on the $\lambda$ term, i.e., state-distribution mismatch accumulating over steps, is an error source that LM maps do not have (constants in Table~\ref{tab:bounds}, proofs in Appendix~\ref{app:proofs}).

\begin{table}[t]
  \centering\small
  \caption{Error bounds by source, with constants: approximation error (Theorem~\ref{thm:approx}, top three rows) and sampling error (Corollary~\ref{thm:sample}, bottom row, which defines $u_N$). The trajectory bound is about $T$ times the step bound for every source.}
  \label{tab:bounds}
  \setlength{\tabcolsep}{3pt}
  \renewcommand{\arraystretch}{1.2}
  \begin{tabular}{@{}llll@{}}
    \toprule
    Source & Step & Trajectory & Ratio \\
    \midrule
    Local expansion & $\tfrac16\varepsilon^3\E_i[T]$ & $\tfrac14\varepsilon^3\E_i[T^2]$ & $\tfrac{3\E_i[T^2]}{2\E_i[T]}$ \\
    Ref.\ mismatch & $\tfrac12\lambda\Tmax\varepsilon^2$ & $\tfrac12\lambda\Tmax^2\varepsilon^2$ & $\Tmax$ \\
    Centering & $\mu^2/(2\bar T)$ & $\mu^2/2$ & $\bar T$ \\
    \midrule
    Sampling & $\tfrac{6\Tmax-1}{2}\varepsilon^2u_N$ & $\tfrac52\Tmax^2\varepsilon^2u_N$ & $\approx\tfrac56\Tmax$ \\
    \bottomrule
  \end{tabular}
\end{table}

\subsection{Total error: approximation plus sampling}
\label{sec:finite}
Theorem~\ref{thm:approx} concerns the population distances. In practice, $\hat D_S$ and $\hat D_\tau$ are computed from $N$ reference trajectories, so the total error adds sampling error to the approximation error. Bounding $\E_i[T]\le\Tmax$ and $\E_i[T^2]\le\Tmax^2$ in Table~\ref{tab:bounds} gives the following.
\begin{corollary}[Total error]
\label{thm:sample}
Let $\tau_1,\dots,\tau_N\sim Q$ be i.i.d., and let $\hat{\bar T}:=M/N$, $\hat K_S:=\tfrac{\hat{\bar T}}{2}\hat D_S^2$, $\hat K_\tau:=\tfrac12\hat D_\tau^2$, and $u_N:=\sqrt{\log(8/\delta)/(2N)}$. Under Assumptions~1 and~2, with probability at least $1-\delta$, simultaneously
\begin{align*}
\abs{K_{ij}-\hat K_S}&\le\Tmax\,\varepsilon^2\bigl(\tfrac{\varepsilon}{4}+\tfrac{\lambda}{2}+3u_N\bigr)+\tfrac{\mu^2}{2\bar T},\\
\abs{K_{ij}-\hat K_\tau}&\le\Tmax^2\varepsilon^2\bigl(\tfrac{\varepsilon}{4}+\tfrac{\lambda}{2}+3u_N\bigr)+\tfrac{\mu^2}{2}.
\end{align*}
\end{corollary}
The trajectory bound carries an extra factor of $T$ on every term. Since $K_{ij}\le\varepsilon^2\E_i[T]/2$ by Lemma~\ref{lem:local}(iii) and the chain rule, the step bound grows with $T$ like the KL it estimates, whereas the trajectory bound grows as $T^2$: summing log-ratios before squaring picks up cross terms between steps (Appendix~\ref{app:cors}). Sampling error shares the parentheses, so it keeps this factor; it decays as $N^{-1/2}$, while the other terms do not. The proof applies Hoeffding's inequality~\citep{hoeffding1963} to four sample means (Appendix~\ref{app:finite}).

\paragraph{Scope of the guarantees.}
The guarantees cover only \emph{close pairs}, with small $\varepsilon$ and $\lambda$. We therefore test empirically whether distances preserve the \emph{order} of the KL outside this regime (\S\ref{sec:klfit}) and whether close pairs behave alike when run (\S\ref{sec:behaviour}).

\section{What does the map capture about an agent?}
\label{sec:mapping}

This section fixes the harness to a standard configuration (BASE: a history window of the last four steps, no planning, direct action; \S\ref{sec:harnesssetup}) and places each LLM as one point. The question is what the coordinates capture about an agent. LM maps showed a \emph{correlation} between coordinates and downstream performance~\citep{oyama2025mapping}. We answer with three quantities: the target itself, the KL between the trajectory distributions of the agents when run (\S\ref{sec:klfit}); one projection of it, the success rate (\S\ref{sec:predicting}); and the behavior itself when run (\S\ref{sec:behaviour}).

\subsection{Does distance track the on-policy divergence?}
\label{sec:klfit}
Theorem~\ref{thm:approx} applies only in the local regime of small $\varepsilon$, that is, to nearly identical policies, whereas actual pairs of models may be far apart. We therefore check empirically whether the distances track the trajectory KL outside that regime, and which version, step or trajectory, tracks it better. In the bounds, the step version is favored by a factor of order $T$ (Table~\ref{tab:bounds}); we ask whether this advantage shows in the data. Measuring the trajectory KL directly requires rollouts of every model and $O(n^2)$ cross-scoring, so we use eight models: each solves a common set of ALFWorld~\citep{shridhar2021alfworld} train tasks, every model scores every trajectory, and the cross-scoring gives $K_{\mathrm{sym}}$ for the $\binom{8}{2}=28$ pairs. The distances use the pooled trajectories of the eight models as the common reference (Appendix~\ref{app:klcontrols}).

\begin{figure}[t]
  \centering
  \includegraphics[width=0.85\linewidth]{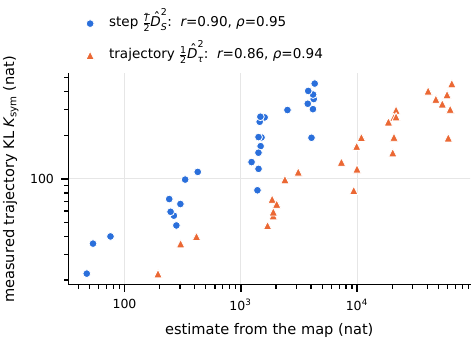}
  \caption{Map estimates, step version $\tfrac{\bar T}{2}\hat D_S^2$ and trajectory version $\tfrac12\hat D_\tau^2$, against the directly measured symmetric trajectory KL $K_{\mathrm{sym}}$: 8 models, 28 pairs, centered coordinates, log--log axes; error bars are standard errors over the 988 shared ALFWorld train tasks. Both track $K_{\mathrm{sym}}$, the step version more closely.}
  \label{fig:klfit}
\end{figure}

\paragraph{Both correlate with the KL, the step version more strongly.}
Figure~\ref{fig:klfit} shows the result (numbers in Table~\ref{tab:klfit}, Appendix~\ref{app:centering}). The centered step estimate $\tfrac{\bar T}{2}\hat D_S^2$ correlates with $K_{\mathrm{sym}}$ at Pearson $r=0.90$ and Spearman $\rho=0.95$, and the trajectory estimate $\tfrac12\hat D_\tau^2$ at $0.86$ / $0.94$: distances measured without running the agents preserve the order of the trajectory KL measured by running them. Two trends are consistent with the bounds (Appendix~\ref{app:centering}): the RMSE of the trajectory version is an order of magnitude larger than that of the step version, as the orders $T$ versus $T^2$ predict, and adding reference trajectories shrinks only the sampling error while the approximation error remains. Centering, although it adds a term to the bound, raises the correlation and lowers the RMSE for both versions. Reading the distances between the eight models from the 100-model map of \S\ref{sec:predicting} instead, with different reference trajectories and no grammar constraint, gives $r=0.86$ and $\rho=0.88$ (Appendix~\ref{app:centering}). Both versions overestimate the measured KL in level, by a median factor of 7.5 for the step version and 74 for the trajectory version (Appendix~\ref{app:centering}), so in what follows we use the distances only as a proxy for the order of the KL, and we use the step version.

\subsection{Performance: predicting without running}
\label{sec:predicting}
\paragraph{Setup.}
From here on we use 100 open-weight LLMs. The reference trajectories are 500 trajectories with 10{,}388 steps, generated by two models (Qwen3.8-27B~\citep{qwen2026qwen38} and Gemma4-E4B~\citep{gemmateam2026gemma4}) on 250 ALFWorld train tasks; every model scores them, and its centered step coordinates are computed from these scores. Only the two generators ran the reference trajectories; the other models only score them. Success rates and action sequences are obtained by running each model itself on different tasks, so the material for the coordinates and the material for the evaluation do not overlap. Scoring uses raw log-likelihoods without the grammar constraint; rescoring with the constraint barely moves the coordinates (Appendix~\ref{app:base100}).

\paragraph{Question.}
Can the downstream success rate of a model that has only computed likelihoods on the reference trajectories, and has never been run in the environment, be predicted from its coordinates? We hold out one model at a time and predict it from the remaining 99. The downstream tasks are 50 ALFWorld valid\_unseen tasks, the environment of the coordinates, and 50 tasks each from \emph{different environments}, TextWorld~\citep{cote2018textworld}, BFCL~\citep{patil2025bfcl}, and MiniWoB++~\citep{shi2017wob,liu2018miniwobpp}, plus the averages over the four benchmarks and over the three other environments (Appendix~\ref{app:predict}). The baselines use only information \emph{available without running}: a regression on the log parameter count (size), the mean success rate of the same family (family), and a regression on the mean log-likelihood $\bar\ell$ (the component removed by centering). On the coordinate side we use kNN, Ridge~\citep{hoerl1970ridge}, and Ridge with size as a covariate, with hyperparameters chosen by inner leave-one-out (Appendix~\ref{app:predict}).

\begin{table}[t]
  \centering\small
  \caption{Predicting success rates without running the agents: Pearson $r$ under leave-one-model-out over 100 models. Baselines use size ($\log$ parameter count), the family mean, or $\bar\ell$ (mean log-likelihood); coordinate methods are kNN, Ridge, and Ridge with size as a covariate (+size). Bold: higher than size, with a bootstrap 95\% interval of the difference excluding zero. RMSE and all predictors in Table~\ref{tab:pred_full}.}
  \label{tab:pred}
  \setlength{\tabcolsep}{3.5pt}
  \footnotesize
  \begin{tabular}{@{}l ccc ccc@{}}
    \toprule
    & \multicolumn{3}{c}{Baselines} & \multicolumn{3}{c}{Coordinates} \\
    \cmidrule(lr){2-4}\cmidrule(lr){5-7}
    Downstream task & Size & Family & $\bar\ell$ & kNN & Ridge & +size \\
    \midrule
    ALFWorld & .63 & .46 & .20 & \textbf{.88} & \textbf{.81} & \textbf{.82} \\
    TextWorld & .71 & .44 & .33 & .76 & .77 & \textbf{.78} \\
    BFCL & .48 & .34 & .32 & \textbf{.69} & \textbf{.69} & \textbf{.65} \\
    MiniWoB++ & .64 & .24 & .22 & .28 & .56 & .63 \\
    \midrule
    4-bench avg. & .74 & .42 & .34 & .74 & .78 & \textbf{.83} \\
    3 other envs. & .72 & .40 & .34 & .72 & .75 & \textbf{.81} \\
    \bottomrule
  \end{tabular}
\end{table}

\paragraph{Results.}
Table~\ref{tab:pred} reports Pearson $r$. In ALFWorld, the environment of the coordinates, kNN reaches $r=0.88$ against $0.63$ for size (the bootstrap interval of the difference excludes zero): the success rate of a model that has only run forward passes on the 500 reference trajectories is predicted from the results of the other 99 models with an RMSE of 9 points on the 50 valid\_unseen tasks. In the other environments, all three coordinate methods beat size on BFCL ($0.65$--$0.69$ versus $0.48$) and TextWorld ($0.76$--$0.78$ versus $0.71$); on MiniWoB++ size is stronger ($0.64$ versus $0.56$ for Ridge), although a quantity built from likelihoods alone still captures this much. Combining size and coordinates matches size on MiniWoB++ ($0.63$) and improves the four-benchmark average to $0.83$ (size $0.74$) and the three-environment average to $0.81$ ($0.72$). Since $\bar\ell$ is weak and the coordinates are strong, what matters is not the overall \emph{level} of the likelihood but its \emph{shape}: which actions a model finds natural in which states. A success rate, however, is only one projection of behavior when run; we ask next whether that is all the coordinates capture.

\subsection{Behavior: do neighbors act alike?}
\label{sec:behaviour}
In \S\ref{sec:predicting}, models that are close in coordinates had similar success rates. Do neighbors also share their \emph{behavior} when run? Behavior can only be measured as a distance between a pair of models, so the question becomes whether the behavioral distance from a model to its coordinate neighbors is smaller than its distance to the remaining models. We look at neighbors rather than at the correlation over all pairs because the guarantee covers only close pairs (\S\ref{sec:finite}). Coordinate neighbors are compared with two kinds of neighbors. The first are neighbors chosen by size, family, or $\bar\ell$, which, as in Table~\ref{tab:pred}, use only information available without running. The second are neighbors chosen by the measured success rate on that benchmark, a reference that is available only after running. Agents with the same success rate do not necessarily solve tasks in the same way, so the gap to the second kind shows what the coordinates capture beyond closeness in performance.

\paragraph{Setup.}
We use the action sequences that each model itself produced with the BASE harness: ALFWorld valid\_unseen (50 tasks), the environment of the coordinates, and the other environments TextWorld (50), BFCL multi\_turn (25), and MiniWoB++ (50), all for the 100 models. The \textbf{behavioral distance} between two models is the distance between their trajectories on the same task, averaged over tasks. Since similar behavior is not a single number, we measure it in four aspects from coarse to fine (definitions in Appendix~\ref{app:behaviour}):
\begin{itemize}[leftmargin=*,itemsep=0pt,topsep=1pt]
  \item \textbf{Composition} (which actions are used how often, ignoring order): Jensen--Shannon distance between the distributions of action types~\citep{endres2003metric} (type JS)
  \item \textbf{Global} (agreement of the whole action sequence): normalized Levenshtein distance~\citep{levenshtein1966,rabanser2026reliability} (edit distance)
  \item \textbf{Local} (what follows a given action): Jaccard distance between the sets of action 2-grams (2-gram)
  \item \textbf{Early} (whether the same first move is made from the same state): disagreement rate of the first action from the same initial state~\citep{bouchard2026trajectory,gonuguntla2026replay} (first action)
\end{itemize}
The difference in success rate is measured in the same form, which puts the results of \S\ref{sec:predicting} on the same scale. The \textbf{effect size} for a model's five coordinate neighbors is the mean behavioral distance to the remaining models minus the mean distance to the neighbors, divided by the standard deviation over all pairs (positive: neighbors more similar). Figure~\ref{fig:map}(c) applies the same definition to the \mbox{$k$-th} nearest model alone, in coordinates and in size, averaged over the 100 models and the four aspects.

\begin{table}[t]
  \centering\small
  \caption{Effect sizes of the behavioral similarity of each model's five neighbors (\S\ref{sec:behaviour}; positive: neighbors more similar than the remaining models), for the difference in success rate and the four aspects of \S\ref{sec:behaviour}. Bold: largest of the four columns other than the reference. $^*$: $p<0.01$, $^+$: $p<0.05$; $^\dagger$: significantly closer than success-rate neighbors ($p<0.05$). The success-rate row is omitted for success-rate neighbors. All measures in Table~\ref{tab:behaviour_full}.}
  \label{tab:behaviour}
  \setlength{\tabcolsep}{3pt}
  \renewcommand{\arraystretch}{1.04}
  \scriptsize
  \begin{tabular}{@{}l ccc c c@{}}
    \toprule
    & \multicolumn{3}{c}{Available without running} & Map & Reference \\
    \cmidrule(lr){2-4}\cmidrule(lr){5-5}\cmidrule(lr){6-6}
    Benchmark / measure & Size & Family & $\bar\ell$ & Coord. & Success \\
    \midrule
    \multicolumn{6}{@{}l}{\emph{ALFWorld}} \\
    \quad Success rate & .41$^{*}$ & .49$^{*}$ & .20$^{*}$ & \textbf{.60$^{*}$} & --- \\
    \addlinespace[1.5pt]
    \quad Composition & .23$^{*}$ & .28$^{*}$ & .05 & \textbf{.52$^{*\dagger}$} & .30$^{*}$ \\
    \quad Global & .40$^{*}$ & .62$^{*}$ & .18$^{+}$ & \textbf{.88$^{*\dagger}$} & .65$^{*}$ \\
    \quad Local & .41$^{*}$ & .68$^{*}$ & .20$^{*}$ & \textbf{.96$^{*\dagger}$} & .74$^{*}$ \\
    \quad Early & .30$^{*}$ & .49$^{*}$ & .11$^{*}$ & \textbf{.61$^{*\dagger}$} & .41$^{*}$ \\
    \addlinespace[3pt]
    \multicolumn{6}{@{}l}{\emph{TextWorld}} \\
    \quad Success rate & .58$^{*}$ & .56$^{*}$ & .27$^{*}$ & \textbf{.61$^{*}$} & --- \\
    \addlinespace[1.5pt]
    \quad Composition & .10 & .27$^{*}$ & .06 & \textbf{.33$^{*\dagger}$} & .22$^{*}$ \\
    \quad Global & .25$^{*}$ & .42$^{*}$ & .14 & \textbf{.52$^{*}$} & .50$^{*}$ \\
    \quad Local & .27$^{+}$ & .44$^{*}$ & .14 & \textbf{.54$^{*}$} & .52$^{*}$ \\
    \quad Early & .14$^{*}$ & \textbf{.25$^{*}$} & .04 & .21$^{*}$ & .16$^{*}$ \\
    \addlinespace[3pt]
    \multicolumn{6}{@{}l}{\emph{BFCL (multi\_turn)}} \\
    \quad Success rate & .27$^{*}$ & .30$^{*}$ & .15$^{*}$ & \textbf{.47$^{*}$} & --- \\
    \addlinespace[1.5pt]
    \quad Composition & .10$^{+}$ & \textbf{.21$^{*}$} & .03 & .18$^{*}$ & .23$^{*}$ \\
    \quad Global & .13$^{*}$ & \textbf{.26$^{*}$} & .06 & .26$^{*}$ & .30$^{*}$ \\
    \quad Local & .21$^{*}$ & .32$^{*}$ & .09$^{+}$ & \textbf{.38$^{*}$} & .39$^{*}$ \\
    \quad Early & .19$^{*}$ & .25$^{*}$ & .09$^{*}$ & \textbf{.32$^{*}$} & .21$^{*}$ \\
    \addlinespace[3pt]
    \multicolumn{6}{@{}l}{\emph{MiniWoB++}} \\
    \quad Success rate & \textbf{.42$^{*}$} & .37$^{*}$ & .20$^{*}$ & .39$^{*}$ & --- \\
    \addlinespace[1.5pt]
    \quad Composition & .12$^{*}$ & .22$^{*}$ & .08$^{*}$ & \textbf{.24$^{*\dagger}$} & .13$^{*}$ \\
    \quad Global & .20$^{*}$ & .32$^{*}$ & .10$^{+}$ & \textbf{.38$^{*}$} & .37$^{*}$ \\
    \quad Local & .27$^{*}$ & .47$^{*}$ & .16$^{+}$ & \textbf{.50$^{*}$} & .42$^{*}$ \\
    \quad Early & .35$^{*}$ & .35$^{*}$ & .09$^{+}$ & \textbf{.50$^{*\dagger}$} & .28$^{*}$ \\
    \bottomrule
  \end{tabular}
\end{table}

\paragraph{Models that behave alike can be found without running them.}
The five coordinate neighbors not only have similar success rates ($0.39$--$0.61$) but also behave alike on all four benchmarks (Table~\ref{tab:behaviour}). The effect is roughly twice that of size neighbors (edit distance on ALFWorld $0.88$ versus $0.40$; 2-gram on MiniWoB++ $0.50$ versus $0.27$), and $\bar\ell$ neighbors are weak ($0.03$--$0.20$). Family is a stronger competitor than size: coordinates exceed it on ALFWorld, TextWorld, and MiniWoB++ (except for the first action on TextWorld) and tie with it on BFCL. The effect remains when the coordinate neighbors are restricted to other families (edit distance on ALFWorld $0.60$), so it is not produced by same-family dependence. Similarity decreases with rank: the first coordinate neighbor is the most similar and the effect returns to the overall mean around rank 30, whereas size neighbors are nearly flat beyond the first (Figure~\ref{fig:map}(c)). The divergence position and the profile of idiosyncrasies show the same pattern, while for the success pattern, neighbors chosen by family or size are about as close as or closer than coordinate neighbors (Appendix~\ref{app:behaviour}).

\paragraph{Beyond closeness in performance.}
In ALFWorld, the environment of the coordinates, coordinate neighbors behave more alike than neighbors chosen by the measured success rate in all four aspects ($\dagger$ in Table~\ref{tab:behaviour}); in the other environments the two are roughly on a par, with coordinates significantly closer only for composition on TextWorld and MiniWoB++ and for the first action on MiniWoB++. We also regress out performance and size (Appendix~\ref{app:behaviour}): because strong models may converge on the shortest procedures, we regress each pair's behavioral distance on the difference in success rate, the pair's mean success rate and its square, and the difference in log parameter count, and measure the neighbor effect on the residuals. In ALFWorld the effect remains (edit distance $0.61$, 2-gram $0.68$) and stays positive when the candidates are restricted to other families ($0.31$, $0.31$); in the other environments it is weak, remains mainly in the local and early measures, and for some measures vanishes or changes sign. This linear control shows only that an association remains that the observed success rates and sizes do not explain. That the coordinates are built from the likelihood of the next action in a fixed state is consistent with the effect remaining in the local and early measures.

\section{Composing Harness Effects in Agent Behavior Maps}
\label{sec:harness}

\begin{figure}[!t]
  \centering
  \includegraphics[width=0.85\linewidth]{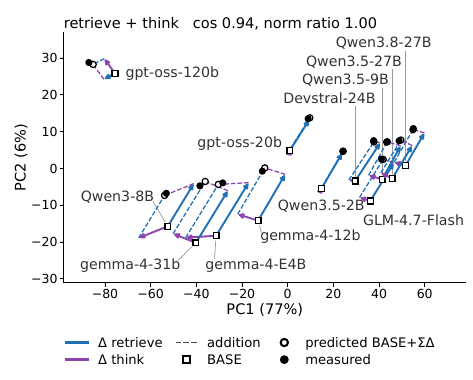}
  \caption{\textbf{Harness displacements add up within an LLM.} The 144 $\tau$-bench agents (12 LLMs $\times$ 12 harnesses) in their first two principal components (PC1 separates LLMs; axes scaled independently). Squares: BASE; solid arrows: single-component displacements; dashed: their sum, ending at the predicted position of retrieve + think (open circles); filled circles: its measured position. Cosine and norm ratio: medians over the 12 LLMs in the full space.}
  \label{fig:additivity}
\end{figure}

Section~\ref{sec:mapping} fixed the harness; this section fixes the LLM, varies the harness, and asks \textbf{what structure the map has with respect to harnesses}. A harness combines components for context management, planning, and acting. If adding a component appears as a displacement within an LLM and these displacements add up, the map reads as the position of the LLM plus a sum of component vectors, and the coordinates of an unscored combination follow from the scores of single components. We test whether the approximate additivity observed for prompt changes in LM maps~\citep{takase2026maps} also holds for agent harnesses (\S\ref{sec:additivity}).

\subsection{Harness interventions and setup}
\label{sec:harnessvec}
\label{sec:harnesssetup}
\paragraph{Harnesses.}A harness splits the rule $C_i$ that builds the LLM's input from the external history $h_t$ (\S\ref{sec:targets}) into three slots.
\begin{itemize}[leftmargin=*,itemsep=0pt,topsep=1pt]
  \item \textbf{Context management} (which part of the history is shown): windowed passes the initial observation and the last four steps verbatim; retrieve adds past steps retrieved for their relevance; summarize replaces old history with a structured summary.
  \item \textbf{Planning} (whether a plan is rewritten before acting): none / replan.
  \item \textbf{Acting} (whether free-form reasoning is written before the action): direct / think.
\end{itemize}
System instructions, tool definitions, and the initial observation are shared by all harnesses; plans and reasoning enter only the current step's input, so the external history consists of actions and observations only and every harness builds its input from the same $h_t$, both when running and when scoring. The $3\times2\times2=12$ combinations include BASE (windowed/none/direct), four \emph{single-component} harnesses that add one component to BASE, and seven \emph{composed} harnesses that add two or more (specifications in Appendix~\ref{app:harnesssetup}).

\paragraph{Harness vectors.}For the same LLM $m$, the difference between the coordinates of BASE and of the harness that adds component $A$, $\Delta^{(m)}_A:=z(m,A)-z(m,\base)$, is the \textbf{harness vector}, and the sum of harness vectors predicts the coordinates of a composition, $\hat z(m,A{+}B):=z(m,\base)+\Delta^{(m)}_A+\Delta^{(m)}_B$.
Since the coordinates are centered log-likelihoods on the common reference trajectories, $\Delta_A$ represents how component $A$ changes the relative likelihood of the reference actions; the change in the mean log-likelihood is reported separately as a scalar (Appendix~\ref{app:harnesssetup}).

\paragraph{Setup.}All 12 configurations are scored, but only the coordinates of the five configurations BASE and single components enter the composition; the coordinates of the remaining seven are constructed as their sums. With 12 LLMs (listed in Appendix~\ref{app:harnesssetup}), the coordinates of 60 agents thus yield composed coordinates for 84. The measured coordinates of the composed harnesses are used only for validation. The coordinates are obtained by teacher forcing on 359 HAL~\citep{kapoor2025hal} traces of $\tau$-bench~\citep{yao2025taubench}; ALFWorld and SWE-Bench (6 LLMs) are in Appendix~\ref{app:harness}.

\subsection{Predicting coordinates of unscored compositions}
\label{sec:additivity}
Does the sum of the single-component displacements approach the measured displacement $z(m,A{+}B)-z(m,\base)$ of the composition? Figure~\ref{fig:additivity} shows one composition, retrieve + think, as an example; all seven are in Figure~\ref{fig:additivity_all} (Appendix~\ref{app:additivity}). For all compositions, Table~\ref{tab:additivity_full} in the appendix reports, within each LLM, the cosine between the measured displacement and the predicted displacement $\Delta_A+\Delta_B$, the norm ratio (measured / predicted), and the relative reconstruction error $\norm{\text{measured}-\text{predicted}}/\norm{\text{measured}}$, as medians over LLMs. We compare displacements from BASE because absolute coordinates share a large component common to the LLM, which alone would make them look close. The arrows of the same component also look parallel across LLMs in Figure~\ref{fig:additivity}; this cross-LLM agreement of directions is not analyzed in this paper.

For two-component compositions, the median cosine over the 60 cases (12 LLMs $\times$ 5 compositions) is $0.90$ and the median norm ratio $0.98$, so the sum roughly matches in direction and length. The agreement varies by composition: retrieve with planning or reasoning agrees best (cosine $0.97$ / $0.94$, relative error $0.25$ / $0.34$), compositions with summarize keep the direction but deviate more (cosine $0.91$ / $0.87$), and replan + think is weakest (cosine $0.81$, norm ratio $0.74$: the measured displacement is shorter than predicted). Three-component compositions reach cosine $0.82$ and norm ratio $0.87$ (median over $2\times12$ cases). On ALFWorld and SWE-Bench (6 LLMs) the agreement weakens to median two-component cosines of $0.63$ and $0.81$ (Table~\ref{tab:additivity_full}). Additivity thus depends on the components and the environment, but for context management combined with planning or reasoning, the direction and magnitude of the composed displacement can be approximated from single-component displacements.
\section{Related Work}
\label{sec:related}

\paragraph{Maps of language models and policies.}
LM maps~\citep{oyama2025mapping} place models in a common coordinate system by their log-likelihoods on fixed texts and have been extended to prompt--response distributions~\citep{takase2026maps}; other works embed models by their outputs or correctness~\citep{wu2025llmdna,yax2024phylolm,zhuang2024embedllm}. All of them address static LMs. In reinforcement learning, policies are compared by their actions on a fixed set of states~\citep{parkerholder2020dvd,pacchiano2020behaviors,konyushkova2021aops}. We extend LM maps to closed-loop agents and quantify the cost of measuring on fixed trajectories as the term $\lambda$, which accumulates over steps like the compounding error of imitation learning~\citep{ross2010efficient,ross2011dagger}. Predicting performance from scores on cheaper benchmarks~\citep{ruan2024observational,song2026pace} shares our goal but needs the target model's scores, whereas our coordinates need neither scores nor environment runs.

\paragraph{Harness effects and agent behavior.}
Harness effects have been measured as changes in success rate by running every LLM--harness combination~\citep{zhang2026stopcomparing,yao2026harnessbench,fan2026harnessdesign,cho2026nonmonotone,chen2026agentspec} or decomposed into LLM and scaffold abilities with IRT~\citep{ge2026psychometrics}; we represent a component as a displacement on the map, obtained without running the composed agent, and ask whether displacements add (\S\ref{sec:harness}). Evaluations increasingly report consistency across trials, cost, and behaviors found in logs alongside success rates~\citep{yao2025taubench,kapoor2025agentsmatter,kapoor2025hal}; harnesses with similar success rates can differ greatly in token use and failure patterns~\citep{vats2026scaffold}, and agent behavior carries fingerprints of the underlying LLM~\citep{sun2025idiosyncrasies,wang2026lidar,oderinwale2026trajectories,lugoloobi2026uitraces}. All of these works run the agents and compare the resulting behavior through aggregate statistics or classifiers. We measure behavioral differences by distances obtained without running the agents, relate them to the trajectory KL with explicit error bounds (\S\ref{sec:method}), and show that pairs at small distance behave alike when run (\S\ref{sec:behaviour}).
\section{Conclusion}
\label{sec:conclusion}
Agent Behavior Maps, obtained by teacher forcing on common reference trajectories, locate agents without running them. Their coordinates predict the success rates of models that were never run, on the benchmark the map was built from and, combined with size, in other environments. Neighbors on the map take similar action sequences when run. The theory decomposes the gap between the distances and the trajectory KL into four error sources and limits the guarantee to close pairs. Harness components appear as displacements within an LLM whose sums approximate the coordinates of unscored combinations, with an accuracy that depends on the component and the environment. The map represents not only a ranking by performance but also how close agents are in what they do.

\section*{Limitations}
\label{sec:limitations}

\paragraph{Locality.}
Both the theoretical guarantee and the empirical correspondence are limited to close pairs. The KL can be read off the distance only for pairs in the local regime of small $\varepsilon$ and $\lambda$, and pairs of different LLMs are not guaranteed to fall in it: the estimates exceed the measured KL by a median factor of 7.5, although their order agrees (\S\ref{sec:klfit}). The correspondence with behavioral distance holds for neighbors and fades beyond about the 30th neighbor (Figure~\ref{fig:map}(c)); we have not shown that the magnitude of the distance between far-apart pairs measures how different their behavior is. The map answers who is similar to whom, not by how much.

\paragraph{Dependence on the reference trajectories.}
Differences between policies in regions that $Q$ never visits cannot be measured, and the reference mismatch $\lambda$ remains a sensitivity parameter that the likelihood matrix alone does not determine. A general theory of the reference policy is left for future work.

\paragraph{Estimation of the KL.}
The trajectory log-likelihood ratios in \S\ref{sec:klfit} are based on the probability of the canonical token sequence and differ from action probabilities that marginalize over all tokenizations. In environments where the candidate set can be enumerated, this can be resolved by renormalization, but in general it is a limitation.

\paragraph{Access to log-likelihoods.}
The coordinates require the target model's log-likelihood of a given action sequence. Models that are available only through an API that does not return the log-probabilities of given text cannot be placed on the map, so this study is limited to open-weight models.

\paragraph{Sample size of behavior.}
The behavioral comparisons are based on one rollout per task, and the behavior of the same model varies from run to run~\citep{bjarnason2026randomness,gonuguntla2026replay}. We have not compared the size of this variation with the behavioral distance to neighbors.

\paragraph{Scope of the harness experiment.}
The harnesses in \S\ref{sec:harness} are limited to the 12 configurations of our own three-slot design (context management, planning, acting) and to 12 LLMs; the main additivity results are on $\tau$-bench, and ALFWorld and SWE-Bench cover 6 LLMs. Whether the same additivity holds for other scaffolds with differently designed components has not been tested.

\bibliography{refs}

\appendix
\renewcommand{\topfraction}{0.95}
\renewcommand{\bottomfraction}{0.6}
\renewcommand{\dbltopfraction}{0.95}
\renewcommand{\textfraction}{0.05}
\renewcommand{\floatpagefraction}{0.9}
\renewcommand{\dblfloatpagefraction}{0.9}
\setcounter{topnumber}{3}
\setcounter{dbltopnumber}{2}
\setcounter{totalnumber}{5}
\section{Theory: proofs and scoring details}
\label{app:theory}
This appendix supplements \S\ref{sec:method}. Appendix~\ref{app:probability} defines the policy and the scoring convention; Appendices~\ref{app:proofs}--\ref{app:uncentered} prove Theorem~\ref{thm:approx}, Corollary~\ref{thm:sample}, and the cross-term identity~\eqref{eq:cross}. Notation follows \S\ref{sec:vectors}.

\subsection{Grammar-constrained action probabilities}
\label{app:probability}

We define action probabilities under grammar-constrained decoding. A deployed agent can execute only outputs that match the environment's action format, so its outputs are typically constrained to an action grammar, e.g., by constrained decoding. Our grammar constraint reflects this setting: the constrained policy is the distribution of actions the agent actually sends to the environment.

Logged actions are serialized by a common normalization map $\mathcal N$ that preserves their effect on the environment. Let $\mathcal A(h)$ be the set of action strings that satisfy the grammar under history $h$. Agent $i$ reconstructs its input $C_i(h)$ from $h$ with its rule $C_i$ (\S\ref{sec:method}). Let $q_i(v\mid C_i(h),v_{<k})$ be the model's next-token probability given this input and a prefix $v_{<k}$, and let $\mathcal V(h,v_{<k})$ be the set of tokens allowed by the grammar after $v_{<k}$. Grammar-constrained decoding sets the probabilities of disallowed tokens to zero and renormalizes the rest: for $v_k\in\mathcal V(h,v_{<k})$,
\begin{equation}
q_i^{\mathcal G}(v_k\mid h,v_{<k})=\frac{q_i(v_k\mid C_i(h),v_{<k})}{\sum\limits_{v\in\mathcal V(h,v_{<k})}q_i(v\mid C_i(h),v_{<k})}.
\label{eq:grammar-renorm}
\end{equation}
The probability of an action $a$ is the sum over all token sequences $v$ that normalize to $a$,
\begin{equation}
\pi_i(a\mid h)=\sum_{v:\,\mathcal N(\mathrm{decode}(v))=a}\ \prod_kq_i^{\mathcal G}(v_k\mid h,v_{<k}),
\label{eq:text-policy}
\end{equation}
which satisfies $\sum_{a\in\mathcal A(h)}\pi_i(a\mid h)=1$.

The trajectory-KL comparison in \S\ref{sec:klfit} uses grammar-constrained scoring; the experiments from \S\ref{sec:predicting} onward score without the grammar mask.

\subsection{Proof of Theorem~\ref{thm:approx}}
\label{app:proofs}
Fix a pair $(i,j)$. By the chain rule~\eqref{eq:chain}, $K_{ij}$ is the expected sum of per-step local KLs. Its gap to $\tfrac{\bar T}2\Dstep$ and $\tfrac12\Dtraj$ has three sources:
\begin{enumerate}[label=(\alph*),leftmargin=*,itemsep=1pt,topsep=2pt]
\item the remainder of the quadratic approximation of the local KL (Lemma~\ref{lem:local});
\item taking the quadratic term under the reference $Q$ instead of the on-policy $P_i$ (Lemma~\ref{lem:bounded});
\item centering.
\end{enumerate}
We prove the two lemmas and then the theorem with explicit constants.

\begin{lemma}[Quadratic approximation of the local KL]
\label{lem:local}
Let $p,q$ be distributions with a common support, let $Y=\log(p/q)$ with $\abs Y\le\varepsilon$, and write $Y_\pm=\max(\pm Y,0)$. Then the following hold.

(i) Two-sided bound:
\begin{equation}
-\tfrac16\E_p[Y^3]\le\KL(p\Vert q)-\tfrac12\E_p[Y^2]\le-\tfrac16\E_q[Y^3].
\label{eq:local-two-sided}
\end{equation}

(ii) Size of the remainder:
\begin{align}
\bigl|\KL(p\Vert q)-\tfrac12\E_p[Y^2]\bigr|&\le\tfrac16\max\{\E_p[Y_+^3],\E_q[Y_-^3]\}\notag\\
&\le\tfrac{\varepsilon^3}6.
\label{eq:local-proof}
\end{align}

(iii) Size of the local KL: $0\le\KL(p\Vert q)\le\varepsilon^2/2$.
\end{lemma}
\begin{proof}
Let $g(y)=e^{-y}-1+y-y^2/2$. Since $\E_p[e^{-Y}]=\sum q=1$,
\begin{equation}
\KL(p\Vert q)-\tfrac12\E_p[Y^2]=\E_p[g(Y)],
\label{eq:local-remainder}
\end{equation}
so it suffices to bound the right-hand side. By Taylor's theorem, $g(y)=-\frac{y^3}6e^{-\xi}$ for some $\xi$ between $0$ and $y$. If $y\ge0$ then $e^{-\xi}\in[e^{-y},1]$, and if $y<0$ then $e^{-\xi}\in[1,e^{-y}]$; in either case
\begin{equation}
-\tfrac{y^3}6\le g(y)\le-\tfrac{y^3}6e^{-y}.
\label{eq:local-g}
\end{equation}
Averaging over $p$ and applying the change of measure $\E_p[e^{-Y}f(Y)]=\E_q[f(Y)]$ to the right-hand side, \eqref{eq:local-remainder} gives \eqref{eq:local-two-sided}.

Next, split \eqref{eq:local-g} by the sign of $Y$. For $Y\ge0$, $-\frac16Y_+^3\le g(Y)\le0$; for $Y<0$, $0\le g(Y)\le\frac16e^{-Y}Y_-^3$. Hence, with $A=\E_p[g(Y)\mathbf 1\{Y<0\}]$ and $B=-\E_p[g(Y)\mathbf 1\{Y\ge0\}]$, we have $\E_p[g(Y)]=A-B$ and
\begin{align}
0&\le A\le\tfrac16\E_p[e^{-Y}Y_-^3]=\tfrac16\E_q[Y_-^3],\notag\\
0&\le B\le\tfrac16\E_p[Y_+^3].
\end{align}
As a difference of two nonnegative terms, $\abs{A-B}\le\max\{A,B\}$, which is the first inequality in \eqref{eq:local-proof}; the second follows from $\abs Y\le\varepsilon$. Splitting by sign, rather than bounding $e^{-\xi}$ uniformly by $e^{\varepsilon}$, avoids a factor $e^{\varepsilon}$.

Finally, we bound the local KL. Let $F(s)=\log\E_p[e^{-sY}]$. Then $F(0)=0$, $F(1)=\log\E_p[e^{-Y}]=0$, and $F'(0)=-\E_p[Y]=-\KL(p\Vert q)$. Moreover, $F''(s)$ is the variance of $Y$ under the distribution proportional to $pe^{-sY}$, so $0\le F''(s)\le\varepsilon^2$. By Taylor's theorem, for some $s_0\in(0,1)$,
\begin{equation}
0=F(1)=-\KL(p\Vert q)+\tfrac12F''(s_0),
\end{equation}
which gives $0\le\KL(p\Vert q)\le\varepsilon^2/2$.
\end{proof}

\begin{lemma}[Difference of expectations of a bounded function]
\label{lem:bounded}
If $L\le f\le U$, then $\abs{\E_P[f]-\E_Q[f]}\le(U-L)\,\TV(P,Q)$.
\end{lemma}
\begin{proof}
Trajectories take values in a countable set, so $\TV(P,Q)=\frac12\sum_\tau\abs{P(\tau)-Q(\tau)}$. Let $c=(L+U)/2$; then $\abs{f-c}\le(U-L)/2$, and since $\sum_\tau\{P(\tau)-Q(\tau)\}=0$,
\begin{align}
\abs{\E_P[f]-\E_Q[f]}&=\Bigl|\textstyle\sum_\tau(f-c)(P-Q)(\tau)\Bigr|\notag\\
&\le\tfrac{U-L}2\textstyle\sum_\tau\abs{P(\tau)-Q(\tau)}\notag\\
&=(U-L)\,\TV(P,Q).
\end{align}
\end{proof}

\paragraph{Proof of Theorem~\ref{thm:approx}.}
We show, with explicit constants, that $\abs{K_{ij}-\tfrac{\bar T}2\Dstep}\le R_S$ and $\abs{K_{ij}-\tfrac12\Dtraj}\le R_\tau$, where
\begin{align}
R_S&=\tfrac{\varepsilon^3}{6}\E_i[T]+\tfrac{\lambda\Tmax\varepsilon^2}{2}+\tfrac{\mu^2}{2\bar T},\label{eq:step-bound}\\
R_\tau&=\tfrac{\varepsilon^3}{6}\E_i[T]+\tfrac{\varepsilon^3}{4}\E_i[T(T-1)]\notag\\
&\quad+\tfrac{\lambda\Tmax^2\varepsilon^2}{2}+\tfrac{\mu^2}{2}.\label{eq:trajectory-bound}
\end{align}
The sum of the first two terms of \eqref{eq:trajectory-bound} is $\frac{\varepsilon^3}4\E_i[T^2]-\frac{\varepsilon^3}{12}\E_i[T]\le\frac{\varepsilon^3}4\E_i[T^2]$, which gives the form in \eqref{eq:Rt} and Table~\ref{tab:bounds}.

For $K_{ij}$ only the parts of Assumptions~1 and~2 concerning $P_i$ and $Q$ are used; the parts concerning $P_j$ enter through the symmetric KL below. We set $d_t=0$ after termination, and write $Z=\sum_{t<T}d_t^2$ and $\mathcal F_t=\sigma(h_t)$. Whether the episode is still running at time $t$ is determined by $h_t$, so $\{t<T\}$ is $\mathcal F_t$-measurable.

\paragraph{Step version.}
At each step $t<T$, set $p=\pi_i(\cdot\mid h_t)$ and $q=\pi_j(\cdot\mid h_t)$; then $Y=d_t$ and $\E_p[Y^2]=\E_i[d_t^2\mid\mathcal F_t]$. By Lemma~\ref{lem:local}(ii), $\KL(\pi_i(\cdot\mid h_t)\Vert\pi_j(\cdot\mid h_t))=\frac12\E_i[d_t^2\mid\mathcal F_t]+r_t$ with $\abs{r_t}\le\varepsilon^3/6$. Substituting this into the chain rule~\eqref{eq:chain} and using the tower property gives
\begin{equation}
K_{ij}=\tfrac12\E_iZ+r_*,\qquad\abs{r_*}\le\tfrac{\varepsilon^3}6\E_i[T].
\label{eq:step-expansion}
\end{equation}
On the other hand, $\bar T\Dstep=\E_Q\sum_{t<T}(d_t-\mu/\bar T)^2=\E_QZ-2\frac{\mu}{\bar T}\E_QX+\frac{\mu^2}{\bar T^2}\E_Q[T]=\E_QZ-\mu^2/\bar T$, so
\begin{equation}
K_{ij}-\tfrac{\bar T}2\Dstep=r_*+\tfrac12(\E_iZ-\E_QZ)+\tfrac{\mu^2}{2\bar T}.
\end{equation}
Lemma~\ref{lem:bounded} with $0\le Z\le\Tmax\varepsilon^2$ bounds the second term by $\frac12\lambda\Tmax\varepsilon^2$, and the triangle inequality gives \eqref{eq:step-bound}.

\paragraph{Trajectory version.}
Let $k_t=\E_i[d_t\mid\mathcal F_t]$. For $t<T$, $k_t=\KL(\pi_i(\cdot\mid h_t)\Vert\pi_j(\cdot\mid h_t))$, and for $t\ge T$, $k_t=0$; by Lemma~\ref{lem:local}(iii), $0\le k_t\le\mathbf 1\{t<T\}\,\varepsilon^2/2$. For $s<t$, $d_s$ is $\mathcal F_t$-measurable, so $\E_i[d_sd_t]=\E_i[d_sk_t]$; together with $\abs{d_s}\le\varepsilon$,
\begin{equation}
\Bigl|\textstyle\sum_{s<t}\E_i[d_sd_t]\Bigr|\le\tfrac{\varepsilon^3}2\E_i\bigl[\tbinom T2\bigr]=\tfrac{\varepsilon^3}4\E_i[T(T-1)].
\end{equation}
Since $X^2=Z+2\sum_{s<t}d_sd_t$, we have $\frac12\E_iZ=\frac12\E_iX^2-\sum_{s<t}\E_i[d_sd_t]$; combining this with \eqref{eq:step-expansion} and $\Dtraj=\E_QX^2-\mu^2$ gives
\begin{multline}
K_{ij}-\tfrac12\Dtraj=r_*-\textstyle\sum_{s<t}\E_i[d_sd_t]\\+\tfrac12(\E_iX^2-\E_QX^2)+\tfrac{\mu^2}2.
\end{multline}
Since $0\le X^2\le\Tmax^2\varepsilon^2$, Lemma~\ref{lem:bounded} bounds the third term by $\frac12\lambda\Tmax^2\varepsilon^2$, giving \eqref{eq:trajectory-bound}. The proof assumes neither independence across steps nor independence between the episode length and the actions.

\paragraph{Symmetric KL.}
Swapping $(i,j)$ maps $d_t\mapsto-d_t$, $X\mapsto-X$, and $\mu\mapsto-\mu$, and leaves $Z$, $X^2$, $\bar T$, and hence $\Dstep$ and $\Dtraj$ unchanged (the same $Q$ is used in both directions). The proof holds with $\lambda=\TV(P_i,Q)$, and the bound for $(j,i)$ replaces $\E_i[T]$ and $\TV(P_i,Q)$ by $\E_j[T]$ and $\TV(P_j,Q)$. Averaging the two bounds gives, with $\lambda_{\mathrm{sym}}:=\{\TV(P_i,Q)+\TV(P_j,Q)\}/2\le\lambda$,
\begin{align}
\bigl|K_{\mathrm{sym}}-\tfrac{\bar T}2\Dstep\bigr|&\le\tfrac{\varepsilon^3}{12}(\E_i[T]+\E_j[T])\notag\\
&\quad+\tfrac{\lambda_{\mathrm{sym}}\Tmax\varepsilon^2}2+\tfrac{\mu^2}{2\bar T},
\end{align}
and the trajectory version is analogous. \qed

\subsection{Proof of Corollary~\ref{thm:sample}}
\label{app:finite}
Both $\hat K_S$ and $\hat K_\tau$ are functions of four sample means over the reference trajectories. We bound these means simultaneously by Hoeffding's inequality, propagate the errors to the two distances, and add the approximation error of Theorem~\ref{thm:approx}. As in Appendix~\ref{app:proofs}, Assumption~1 is assumed on the supports of both $P_i$ and $Q$.

\paragraph{The estimators as sample means.}
For each reference trajectory $\tau_r$, let $X_r=\sum_{t<T_r}d_{rt}$ and $Z_r=\sum_{t<T_r}d_{rt}^2$, let an overline denote the sample mean over $r$, and let $\hat{\bar T}=N^{-1}\sum_rT_r$. The difference of the two agents' centering means is $\hat c=\overline X/\hat{\bar T}$, so the distances in \S\ref{sec:vectors} are
\begin{equation}
\hat{\bar T}\hat D_S^2=\overline Z-\overline X^2/\hat{\bar T},\qquad
\hat D_\tau^2=\overline{X^2}-\overline X^2.
\label{eq:sample-exact}
\end{equation}
Likewise, in the population,
\begin{equation}
\bar T\Dstep=\E_QZ-\mu^2/\bar T,\qquad\Dtraj=\E_QX^2-\mu^2
\label{eq:pop-exact}
\end{equation}
(Appendix~\ref{app:proofs}). It thus suffices to bound the sample means of $F\in\{Z,X,X^2,T\}$ simultaneously.

\paragraph{Simultaneous bound on the four sample means.}
Since $\abs{d_t}\le\varepsilon$ and $1\le T\le\Tmax$, the statistics have range widths
\begin{align}
w_Z&=\Tmax\varepsilon^2, & w_X&=2\Tmax\varepsilon,\notag\\
w_{X^2}&=\Tmax^2\varepsilon^2, & w_T&=\Tmax-1.
\label{eq:ranges}
\end{align}
The independent unit is a whole reference trajectory; steps within a trajectory are not treated as independent samples. As the reference trajectories are i.i.d., Hoeffding's inequality gives, for each $F$ with probability at least $1-\delta/4$,
\begin{equation}
\abs{\overline F-\E_QF}\le b_F:=w_Fu_N,
\label{eq:hoeffding}
\end{equation}
with $u_N=\sqrt{\log(8/\delta)/(2N)}$ as in Corollary~\ref{thm:sample}. By the union bound, all four hold with probability at least $1-\delta$, without independence among the statistics. We work on this event.

\paragraph{Step version.}
By \eqref{eq:sample-exact} and \eqref{eq:pop-exact},
\begin{equation}
\hat{\bar T}\hat D_S^2-\bar T\Dstep=(\overline Z-\E_QZ)-\Bigl(\frac{\overline X^2}{\hat{\bar T}}-\frac{\mu^2}{\bar T}\Bigr).
\end{equation}
The second term decomposes by the identity
\begin{align}
\frac{\overline X^2}{\hat{\bar T}}-\frac{\mu^2}{\bar T}
&=\Bigl(\frac{\overline X}{\hat{\bar T}}+\frac\mu{\bar T}\Bigr)(\overline X-\mu)\notag\\
&\quad-\frac{\overline X\mu}{\hat{\bar T}\bar T}(\hat{\bar T}-\bar T).
\end{align}
Since $\abs{\overline X}\le\varepsilon\hat{\bar T}$ and $\abs\mu\le\varepsilon\bar T$, the two coefficients are at most $2\varepsilon$ and $\varepsilon^2$, so \eqref{eq:hoeffding} gives
\begin{equation}
\abs{\hat{\bar T}\hat D_S^2-\bar T\Dstep}\le b_Z+2\varepsilon b_X+\varepsilon^2b_T.
\label{eq:sample-step}
\end{equation}

\paragraph{Trajectory version.}
Similarly, $\hat D_\tau^2-\Dtraj=(\overline{X^2}-\E_QX^2)-(\overline X^2-\mu^2)$, and $\abs{\overline X^2-\mu^2}=\abs{\overline X+\mu}\,\abs{\overline X-\mu}\le2\Tmax\varepsilon b_X$, so
\begin{equation}
\abs{\hat D_\tau^2-\Dtraj}\le b_{X^2}+2\Tmax\varepsilon b_X.
\label{eq:sample-traj}
\end{equation}

\paragraph{Conclusion.}
Since $\hat K_S=\frac12\hat{\bar T}\hat D_S^2$ and $\hat K_\tau=\frac12\hat D_\tau^2$, the triangle inequality, Theorem~\ref{thm:approx}, \eqref{eq:sample-step}, and \eqref{eq:sample-traj} give $\abs{K_{ij}-\hat K_S}\le R_S+B_Su_N$ and $\abs{K_{ij}-\hat K_\tau}\le R_\tau+B_\tau u_N$. Substituting \eqref{eq:ranges},
\begin{align}
B_S&=\tfrac12w_Z+\varepsilon w_X+\tfrac12\varepsilon^2w_T=\tfrac{6\Tmax-1}2\varepsilon^2,\notag\\
B_\tau&=\tfrac12w_{X^2}+\Tmax\varepsilon w_X=\tfrac52\Tmax^2\varepsilon^2.
\end{align}
Finally, $\E_i[T]\le\Tmax$, $\E_i[T^2]\le\Tmax^2$, $\frac{\varepsilon}6\le\frac{\varepsilon}4$, $\frac{6\Tmax-1}2\le3\Tmax$, and $\frac52\le3$ give
\begin{align*}
R_S+B_Su_N&\le\Tmax\varepsilon^2\bigl(\tfrac{\varepsilon}4+\tfrac\lambda2+3u_N\bigr)+\tfrac{\mu^2}{2\bar T},\\
R_\tau+B_\tau u_N&\le\Tmax^2\varepsilon^2\bigl(\tfrac{\varepsilon}4+\tfrac\lambda2+3u_N\bigr)+\tfrac{\mu^2}2,
\end{align*}
which is Corollary~\ref{thm:sample}. Swapping $(i,j)$ leaves $Z$, $X^2$, $T$, and $\abs{\overline X-\mu}$ unchanged, so the same bound holds for $K_{\mathrm{sym}}$. \qed

\subsection{Cross terms in the trajectory distance}
\label{app:cors}

The step distance squares the step log-ratios before summing them; the trajectory distance sums them first. The following identity shows exactly what this adds. Let $c=\mu/\bar T$, $\tilde d_t=d_t-c$ for $t<T$, and $S=\sum_{t<T}\tilde d_t$. Then
\begin{multline}
\Dtraj=\bar T\Dstep+2\E_Q\textstyle\sum_{s<t<T}\tilde d_s\tilde d_t\\
+2c\,\mathrm{Cov}_Q(S,T)+c^2\mathrm{Var}_Q(T).
\label{eq:cross}
\end{multline}
The right-hand side consists of the step distance, the cross term between time steps, and two terms due to variable trajectory length. The cross term has up to $\binom T2$ summands per trajectory and accounts for the extra factor of $T$ in Corollary~\ref{thm:sample}.

\begin{proof}
From $X=\sum_{t<T}d_t=S+cT$ and $\mu=c\bar T$,
\begin{equation}
X-\mu=S+c(T-\bar T).
\label{eq:cross-decomp}
\end{equation}
Since $\Dtraj=\mathrm{Var}_Q(X)=\E_Q(X-\mu)^2$, squaring \eqref{eq:cross-decomp} and averaging over $Q$ gives
\begin{equation}
\begin{split}
\Dtraj=\E_QS^2&+2c\,\E_Q[S(T-\bar T)]\\
&+c^2\E_Q(T-\bar T)^2.
\end{split}
\label{eq:cross-square}
\end{equation}
Here $\E_QS=\E_QX-c\,\E_Q[T]=\mu-\mu=0$, so $\E_Q[S(T-\bar T)]=\mathrm{Cov}_Q(S,T)$ and $\E_Q(T-\bar T)^2=\mathrm{Var}_Q(T)$. The first term expands as
\begin{equation}
\E_QS^2=\E_Q\textstyle\sum_{t<T}\tilde d_t^2+2\E_Q\sum_{s<t<T}\tilde d_s\tilde d_t,
\end{equation}
and by the definition of $\Dstep$ the first term on the right equals $\bar T\Dstep$. Substituting these into \eqref{eq:cross-square} gives \eqref{eq:cross}.
\end{proof}

For example, suppose each reference trajectory has a constant step log-ratio $d_t=b$, where $b$ varies across trajectories with $\E_Qb=0$ and is independent of $T$. Then $\mu=c=0$, $\Dstep=\E_Q[Tb^2]/\bar T=\mathrm{Var}_Q(b)$, and $\Dtraj=\E_Q[T^2]\,\mathrm{Var}_Q(b)$: with strongly correlated steps, the trajectory estimate $\frac12\Dtraj$ is $\E_Q[T^2]/\bar T$ times the step estimate $\frac{\bar T}2\Dstep$. Since the cross term can also be negative, the trajectory version does not always overestimate.

\subsection{Uncentered variants}
\label{app:uncentered}

The uncentered embeddings are $z^S_{i,u}=M^{-1/2}(\ell_{i,rt})_{r,t}$ and $z^\tau_{i,u}=N^{-1/2}(A_{i,r})_r$, with distances $\hat D_{S,u}$ and $\hat D_{\tau,u}$. In the notation of Appendix~\ref{app:finite},
\begin{equation}
\hat{\bar T}\hat D_{S,u}^2=\overline Z,\qquad\hat D_{\tau,u}^2=\overline{X^2},
\label{eq:uncentered-sample}
\end{equation}
and the population distances are $\bar TD_{S,u}^2=\E_QZ$ and $D_{\tau,u}^2=\E_QX^2$. The estimators are $\hat K_{S,u}=\frac12\hat{\bar T}\hat D_{S,u}^2=\frac12\overline Z$ and $\hat K_{\tau,u}=\frac12\hat D_{\tau,u}^2=\frac12\overline{X^2}$. They differ from the centered versions only by the $\mu^2$ terms in \eqref{eq:pop-exact}.

\begin{proposition}[Error of the uncentered variants]
\label{prop:uncentered}
Under the assumptions of Theorem~\ref{thm:approx}, $\abs{K_{ij}-\frac{\bar T}2D_{S,u}^2}\le R_{S,u}$ and $\abs{K_{ij}-\frac12D_{\tau,u}^2}\le R_{\tau,u}$, where
\begin{align}
R_{S,u}&=\tfrac{\varepsilon^3}{6}\E_i[T]+\tfrac{\lambda\Tmax\varepsilon^2}{2},\label{eq:step-bound-u}\\
R_{\tau,u}&=\tfrac{\varepsilon^3}{6}\E_i[T]+\tfrac{\varepsilon^3}{4}\E_i[T(T-1)]\notag\\
&\quad+\tfrac{\lambda\Tmax^2\varepsilon^2}{2}.\label{eq:trajectory-bound-u}
\end{align}
Furthermore, under the assumptions of Corollary~\ref{thm:sample}, with $v_N=\sqrt{\log(4/\delta)/(2N)}$, with probability at least $1-\delta$ simultaneously,
\begin{align}
\abs{K_{ij}-\hat K_{S,u}}&\le R_{S,u}+\tfrac12\Tmax\varepsilon^2v_N,\notag\\
\abs{K_{ij}-\hat K_{\tau,u}}&\le R_{\tau,u}+\tfrac12\Tmax^2\varepsilon^2v_N.
\label{eq:sample-u}
\end{align}
\end{proposition}
\begin{proof}
The approximation error follows Appendix~\ref{app:proofs}. For the step version, \eqref{eq:step-expansion} gives
\begin{align}
K_{ij}-\tfrac{\bar T}2D_{S,u}^2&=K_{ij}-\tfrac12\E_QZ\notag\\
&=r_*+\tfrac12(\E_iZ-\E_QZ)
\end{align}
with $\abs{r_*}\le\frac{\varepsilon^3}6\E_i[T]$, and Lemma~\ref{lem:bounded} bounds the second term by $\frac12\lambda\Tmax\varepsilon^2$, giving \eqref{eq:step-bound-u}. For the trajectory version, $\frac12\E_iZ=\frac12\E_iX^2-\sum_{s<t}\E_i[d_sd_t]$ gives
\begin{multline}
K_{ij}-\tfrac12D_{\tau,u}^2=r_*-\textstyle\sum_{s<t}\E_i[d_sd_t]\\+\tfrac12(\E_iX^2-\E_QX^2).
\end{multline}
The cross term is at most $\frac{\varepsilon^3}4\E_i[T(T-1)]$ by Lemma~\ref{lem:local}(iii), and the third term at most $\frac12\lambda\Tmax^2\varepsilon^2$ by Lemma~\ref{lem:bounded}, giving \eqref{eq:trajectory-bound-u}. Both bounds are the centered ones without the centering terms $\mu^2/(2\bar T)$ and $\mu^2/2$.

For the sampling error, $\hat{\bar T}$ cancels in \eqref{eq:uncentered-sample}, so each version depends on a single sample mean, $\overline Z$ or $\overline{X^2}$. By Hoeffding's inequality, $\abs{\overline Z-\E_QZ}\le w_Zv_N$ and $\abs{\overline{X^2}-\E_QX^2}\le w_{X^2}v_N$ each hold with probability at least $1-\delta/2$, and both together with probability at least $1-\delta$. The triangle inequality then gives \eqref{eq:sample-u}. Since swapping $(i,j)$ leaves $Z$ and $X^2$ unchanged, the bounds also hold for the symmetric KL.
\end{proof}

\section{Mapping and predicting: details and robustness}
\label{app:mapping}
This appendix supplements \S\ref{sec:mapping}. Appendix~\ref{app:base100} describes the setup of the two experiments, and Appendices~\ref{app:centering}, \ref{app:predict}, and \ref{app:behaviour} give details of the trajectory-KL comparison (\S\ref{sec:klfit}), performance prediction (\S\ref{sec:predicting}), and behavior (\S\ref{sec:behaviour}).

\subsection{Setup of the two experiments}
\label{app:base100}
\label{app:klcontrols}
\begin{table*}[!t]
  \centering\small
  \caption{The two experiments of \S\ref{sec:mapping}. Both fix the BASE harness. The multi-harness experiment of \S\ref{sec:harness} is in Table~\ref{tab:hsetup}.}
  \label{tab:setups}
  \setlength{\tabcolsep}{5pt}
  \renewcommand{\arraystretch}{1.2}
  \begin{tabularx}{\textwidth}{@{}>{\raggedright\arraybackslash}p{8.2em}>{\raggedright\arraybackslash}X>{\raggedright\arraybackslash}X@{}}
    \toprule
    & \textbf{Trajectory KL} (\S\ref{sec:klfit}) & \textbf{100 LLMs} (\S\ref{sec:predicting}, \ref{sec:behaviour}) \\
    \midrule
    LLMs & 8 & 100 \\
    Reference trajectories & Each model's own rollouts on the 988 of 1{,}000 ALFWorld train tasks completed by all 8 models (7{,}904 trajectories, 187{,}961 steps) & Rollouts of Qwen3.8-27B and Gemma4-E4B on 250 ALFWorld train tasks (500 trajectories, 10{,}388 steps) \\
    Reference distribution $Q$ & Equal-weight pool of the 8 models & Equal-weight mixture of the 2 generators \\
    Scoring & Grammar-constrained, canonical tokens; every model scores every trajectory ($8\times8$) & No grammar mask; each model's chat template \\
    What is compared & Directly measured $K_{\mathrm{sym}}$ (28 pairs) & Success rates and action sequences on ALFWorld valid\_unseen, TextWorld, BFCL, MiniWoB++ (50 tasks each; behavior on BFCL uses multi\_turn 25); one run per task at temperature 0.7 \\
    \bottomrule
  \end{tabularx}
\end{table*}
Table~\ref{tab:setups} summarizes the two experiments of \S\ref{sec:mapping}; the multi-harness experiment of \S\ref{sec:harness} is summarized in Table~\ref{tab:hsetup}.
\paragraph{Models.}
We use 100 open-weight LLMs, all with the BASE harness on ALFWorld; the full list with parameter counts $|\theta_i|$ is in Table~\ref{tab:models}; parameter counts are total counts, including for MoE models, and the size baselines use their logarithm. Model families, used only as a baseline in \S\ref{sec:predicting}--\ref{sec:behaviour}, are assigned from the model name: 19 families, the largest being Qwen with 20 models.

\paragraph{Effect of the grammar mask.}
The 100-model experiment scores without the grammar mask of Appendix~\ref{app:probability}. To check that this choice does not change the map, we use the scoring runs of ten LLMs (Qwen3-8B, Qwen3.8-27B, Qwen3.5-2B/9B, Gemma4-E4B/12B, GPT-OSS-20B, GLM-4.7-Flash, Devstral-Small-2-24B, Granite-3.3-2B) that recorded both the masked and the raw log-likelihood of each reference action. Six of them (Qwen3-8B, Qwen3.8-27B, Qwen3.5-9B, Gemma4-E4B, GPT-OSS-20B, GLM-4.7-Flash) share 3{,}412 reference steps with both scores. On these steps, the mask changes the mean log-likelihood per action by $-0.01$ to $+0.03$ nats, the per-step values correlate at $r\ge0.99$, and the pairwise squared distances between the six models correlate at $1.000$ with the same ordering of all pairs, with and without centering. For each of the ten LLMs separately, we then replace its raw scores in the 100-model map by the masked scores on the steps where both exist (591--10{,}304 of the 10{,}388 steps, depending on the LLM). This moves its centered coordinate by $0.05$--$0.21$ times the distance to its nearest neighbor (median $0.07$ over the ten LLMs); the nearest neighbor is unchanged for all ten, the five nearest neighbors for nine (four of five for Granite-3.3-2B), and the rank correlation of the distances to the other 99 models is at least $0.997$.

\paragraph{Definition of the trajectory KL.}
The eight models are Qwen3-8B, Gemma4-E4B, GPT-OSS-20B, Qwen3.8-27B, GPT-OSS-120B, GLM-4.7-Flash, Qwen3.5-9B, and Gemma4-12B. $K_{ij}$ is the sample mean of $\log P_i/P_j$ over the trajectories generated by $i$, and $K_{\mathrm{sym}}$ averages the two directions. Negative log-ratios would be kept, but none occur in this set. Since the reference is the equal-weight pool of the eight models, $dP_i/dQ\le8$ by construction.

\paragraph{Compute and reproducibility.}
All likelihood computations and rollouts were run on ten NVIDIA RTX PRO 6000 Blackwell GPUs (96\,GB each) served through vLLM. The complete set of experiments, including the scoring of 10{,}388 reference steps by 100 LLMs, the downstream rollouts of 100 LLMs on four benchmarks, and the $12\times12$ harness experiment, took approximately two weeks of wall-clock time on this cluster; we did not record per-experiment GPU hours. Each downstream task was run once per agent (Table~\ref{tab:setups}), and the hyperparameters of the predictors are listed in Appendix~\ref{app:predict} and those of the harnesses in Appendix~\ref{app:harnesssetup}. Figure~\ref{fig:map}(a) is a t-SNE layout~\citep{vandermaaten2008tsne} of the centered coordinates with perplexity 15; all neighbors in the paper are computed in the original coordinate space. The $\tau$-bench and SWE-Bench reference trajectories are the publicly released HAL traces~\citep{kapoor2025hal}.
\begin{table*}[!htbp]
  \centering\scriptsize
  \caption{The 100 open-weight LLMs of \S\ref{sec:predicting}--\ref{sec:behaviour}, identified by their Hugging Face repository, with parameter counts in billions (total parameters for MoE models) and the family labels used as a baseline. $\diamond$: the eight models of the trajectory-KL comparison (\S\ref{sec:klfit}); $\star$: the twelve LLMs of the harness experiment (\S\ref{sec:harness}).}
  \label{tab:models}
  \setlength{\tabcolsep}{3pt}
  \renewcommand{\arraystretch}{1.02}
  \begin{tabular}{@{}l r l @{\hspace{14pt}} l r l@{}}
    \toprule
    Model & B & Family & Model & B & Family \\
    \midrule
    CohereForAI/c4ai-command-r7b-12-2024 & 8.0 & Cohere & allenai/Llama-3.1-Tulu-3-8B-SFT & 8.0 & Llama \\
    deepseek-ai/deepseek-coder-1.3b-instruct & 1.3 & DeepSeek & codellama/CodeLlama-13b-Instruct-hf & 13.0 & Llama \\
    deepseek-ai/deepseek-coder-6.7b-instruct & 6.7 & DeepSeek & codellama/CodeLlama-34b-Instruct-hf & 33.7 & Llama \\
    deepseek-ai/DeepSeek-Coder-V2-Lite-Instruct & 15.7 & DeepSeek & mistralai/Ministral-3-3B-Instruct-2512 & 3.8 & Mistral \\
    deepseek-ai/DeepSeek-V2-Lite-Chat & 15.7 & DeepSeek & mistralai/Mistral-7B-Instruct-v0.2 & 7.2 & Mistral \\
    deepseek-ai/deepseek-coder-33b-instruct & 33.3 & DeepSeek & mistralai/Mistral-7B-Instruct-v0.3 & 7.2 & Mistral \\
    LGAI-EXAONE/EXAONE-3.5-2.4B-Instruct & 2.4 & EXAONE & mistralai/Ministral-3-8B-Instruct-2512 & 8.9 & Mistral \\
    LGAI-EXAONE/EXAONE-3.5-7.8B-Instruct & 7.8 & EXAONE & mistralai/Mistral-Nemo-Instruct-2407 & 12.2 & Mistral \\
    LGAI-EXAONE/EXAONE-3.5-32B-Instruct & 32.0 & EXAONE & mistralai/Ministral-3-14B-Instruct-2512 & 13.9 & Mistral \\
    tiiuae/Falcon-H1-1.5B-Instruct & 1.6 & Falcon & mistralai/Mistral-Small-24B-Instruct-2501 & 23.6 & Mistral \\
    tiiuae/Falcon3-3B-Instruct & 3.2 & Falcon & mistralai/Devstral-Small-2505 & 23.6 & Mistral \\
    tiiuae/Falcon3-7B-Instruct & 7.5 & Falcon & mistralai/Mistral-Small-3.1-24B-Instruct-2503 & 24.0 & Mistral \\
    tiiuae/Falcon-H1-7B-Instruct & 7.6 & Falcon & mistralai/Devstral-Small-2-24B-Instruct-2512$^{\star}$ & 24.0 & Mistral \\
    tiiuae/Falcon3-10B-Instruct & 10.3 & Falcon & nvidia/NVIDIA-Nemotron-Nano-9B-v2 & 8.9 & Nemotron \\
    google/gemma-3-1b-it & 1.0 & Gemma & nvidia/NVIDIA-Nemotron-3-Nano-30B-A3B-BF16 & 31.6 & Nemotron \\
    google/gemma-3-4b-it & 4.3 & Gemma & allenai/Olmo-3-7B-Instruct & 7.3 & OLMo \\
    google/gemma-4-E2B-it & 5.1 & Gemma & allenai/Olmo-3-7B-Instruct-SFT & 7.3 & OLMo \\
    google/gemma-3n-E4B-it & 7.8 & Gemma & allenai/Olmo-3-7B-Instruct-DPO & 7.3 & OLMo \\
    google/gemma-4-E4B-it$^{\diamond\star}$ & 8.0 & Gemma & allenai/Olmo-3.1-32B-Instruct & 32.2 & OLMo \\
    google/gemma-4-12b-it$^{\diamond\star}$ & 12.0 & Gemma & microsoft/Phi-3-mini-128k-instruct & 3.8 & Phi \\
    google/gemma-3-12b-it & 12.2 & Gemma & microsoft/Phi-3.5-mini-instruct & 3.8 & Phi \\
    google/gemma-4-26B-A4B-it & 25.8 & Gemma & microsoft/Phi-4-mini-instruct & 3.8 & Phi \\
    google/gemma-3-27b-it & 27.4 & Gemma & microsoft/Phi-4-mini-reasoning & 3.8 & Phi \\
    google/gemma-4-31b-it$^{\star}$ & 31.3 & Gemma & microsoft/Phi-3-small-128k-instruct & 7.4 & Phi \\
    THUDM/glm-4-9b-chat & 9.4 & GLM & microsoft/Phi-3-medium-128k-instruct & 14.0 & Phi \\
    zai-org/GLM-4.7-Flash$^{\diamond\star}$ & 31.2 & GLM & microsoft/phi-4 & 14.7 & Phi \\
    openai/gpt-oss-20b$^{\diamond\star}$ & 20.9 & gpt-oss & Qwen/Qwen2.5-0.5B-Instruct & 0.5 & Qwen \\
    openai/gpt-oss-120b$^{\diamond\star}$ & 116.8 & gpt-oss & Qwen/Qwen3-0.6B & 0.8 & Qwen \\
    ibm-granite/granite-3.1-1b-a400m-instruct & 1.3 & Granite & Qwen/Qwen2.5-1.5B-Instruct & 1.5 & Qwen \\
    ibm-granite/granite-3.1-2b-instruct & 2.5 & Granite & Qwen/Qwen3-1.7B & 2.0 & Qwen \\
    ibm-granite/granite-3.3-2b-instruct & 2.5 & Granite & Qwen/Qwen3.5-2B$^{\star}$ & 2.3 & Qwen \\
    ibm-granite/granite-3.1-3b-a800m-instruct & 3.3 & Granite & Qwen/Qwen2.5-3B-Instruct & 3.1 & Qwen \\
    ibm-granite/granite-4.0-h-tiny & 6.9 & Granite & Qwen/Qwen3-4B & 4.0 & Qwen \\
    ibm-granite/granite-3.1-8b-instruct & 8.2 & Granite & Qwen/Qwen3.5-4B & 4.7 & Qwen \\
    ibm-granite/granite-3.3-8b-instruct & 8.2 & Granite & Qwen/Qwen2.5-7B-Instruct & 7.6 & Qwen \\
    tencent/Hunyuan-0.5B-Instruct & 0.5 & Hunyuan & Qwen/Qwen3-8B$^{\diamond\star}$ & 8.2 & Qwen \\
    tencent/Hunyuan-1.8B-Instruct & 1.8 & Hunyuan & Qwen/Qwen3.5-9B$^{\diamond\star}$ & 9.7 & Qwen \\
    tencent/Hunyuan-7B-Instruct & 7.5 & Hunyuan & Qwen/Qwen3-14B & 14.8 & Qwen \\
    internlm/internlm2\_5-1\_8b-chat & 1.9 & InternLM & Qwen/Qwen2.5-14B-Instruct & 14.8 & Qwen \\
    internlm/internlm2\_5-7b-chat & 7.7 & InternLM & Qwen/Qwen2.5-Coder-14B-Instruct & 14.8 & Qwen \\
    internlm/internlm3-8b-instruct & 8.8 & InternLM & Qwen/Qwen3.8-27B$^{\diamond\star}$ & 27.8 & Qwen \\
    LiquidAI/LFM2.5-8B-A1B & 8.5 & LFM & Qwen/Qwen3.5-27B$^{\star}$ & 27.8 & Qwen \\
    meta-llama/Llama-3.2-1B-Instruct & 1.2 & Llama & Qwen/Qwen3-30B-A3B & 30.5 & Qwen \\
    meta-llama/Llama-3.2-3B-Instruct & 3.2 & Llama & Qwen/Qwen3-32B & 32.8 & Qwen \\
    codellama/CodeLlama-7b-Instruct-hf & 6.7 & Llama & Qwen/Qwen2.5-32B-Instruct & 32.8 & Qwen \\
    meta-llama/Llama-3.1-8B-Instruct & 8.0 & Llama & Qwen/Qwen3.6-35B-A3B & 36.0 & Qwen \\
    NousResearch/Hermes-3-Llama-3.1-8B & 8.0 & Llama & HuggingFaceTB/SmolLM3-3B & 3.1 & SmolLM \\
    nvidia/Llama-3.1-Nemotron-Nano-8B-v1 & 8.0 & Llama & 01-ai/Yi-Coder-1.5B-Chat & 1.5 & Yi \\
    allenai/Llama-3.1-Tulu-3-8B & 8.0 & Llama & 01-ai/Yi-1.5-9B-Chat-16K & 8.8 & Yi \\
    allenai/Llama-3.1-Tulu-3-8B-DPO & 8.0 & Llama & 01-ai/Yi-1.5-34B-Chat-16K & 34.4 & Yi \\
    \bottomrule
  \end{tabular}
\end{table*}

\subsection{Trajectory KL: centering, step vs.\ trajectory, and sample size}
\label{app:centering}
Using the eight-model KL data of \S\ref{sec:klfit}, we examine the representation choices fixed in \S\ref{sec:vectors} and the dependence on the number of reference trajectories.
\paragraph{Effect of centering.}
\begin{table}[!htbp]
  \centering\small
  \caption{Agreement between the estimates read from the map and the directly measured symmetric trajectory KL $K_{\mathrm{sym}}$: 8 models, 28 pairs, 988 shared ALFWorld train tasks. The estimates are $\tfrac{\bar T}{2}\hat D_S^2$ for the step version and $\tfrac12\hat D_\tau^2$ for the trajectory version. RMSE is in nats.}
  \label{tab:klfit}
  \setlength{\tabcolsep}{5pt}
  \begin{tabular}{@{}lccr@{}}
    \toprule
    & $r$ & $\rho$ & RMSE \\
    \midrule
    \multicolumn{4}{@{}l}{\textit{Step version}} \\
    \quad Centered & \textbf{0.90} & \textbf{0.95} & \textbf{2{,}056} \\
    \quad Uncentered & 0.83 & 0.92 & 3{,}773 \\
    \multicolumn{4}{@{}l}{\textit{Trajectory version}} \\
    \quad Centered & 0.86 & 0.94 & 29{,}792 \\
    \quad Uncentered & 0.77 & 0.92 & 72{,}014 \\
    \bottomrule
  \end{tabular}
\end{table}

Table~\ref{tab:klfit} compares agreement with the KL with and without centering. On these eight models, centering gives a slightly higher correlation and a smaller RMSE. In performance prediction (Table~\ref{tab:pred_full}), uncentered Ridge is about as good as or slightly better than centered Ridge, while for behavioral neighbors (Table~\ref{tab:behaviour_raw}) centering tends to give larger effects. Since neither choice wins on every measure, the main text reports centered coordinates throughout and the appendix also reports uncentered results.
\paragraph{Step version vs.\ trajectory version.}The step version agrees better with the KL in both correlation and RMSE (Table~\ref{tab:klfit}). Both overestimate the measured $K_{\mathrm{sym}}$, by a median factor of 7.5 for the step version and 74 for the trajectory version; the ratio of the two factors, about 10, is of the order of the mean reference trajectory length of 23.8 steps (Table~\ref{tab:setups}).
\paragraph{Distances on the 100-model map.}The eight models are among the 100 LLMs of \S\ref{sec:predicting}, so their 28 pairwise distances can also be read from the 100-model map, whose reference trajectories come from two of the eight models and are scored without the grammar mask. Against the same $K_{\mathrm{sym}}$, the centered step estimate $\tfrac{\bar T}{2}\hat D_S^2$ from that map gives $r=0.86$ and $\rho=0.88$ (uncentered $0.80$ / $0.92$; centered trajectory version $0.75$ / $0.89$), and the 28 estimates from the two embeddings have rank correlation $0.97$. The lower correlation reflects both the different reference distribution $Q$ and the scoring convention.
\paragraph{Dependence on the number of reference trajectories.}We resample $N$ reference trajectories from the pool of 7{,}904 without replacement, 50 times for each $N$ from 100 to 3{,}200, and recompute the estimates. The resampling standard deviation of the step estimate shrinks roughly as $N^{-1/2}$ (median coefficient of variation over the 28 pairs $0.074$ at $N=100$ and $0.010$ at $N=3{,}200$; log--log slope $-0.48$ after the finite-population correction, $-0.52$ for the trajectory version), consistent with Corollary~\ref{thm:sample}. The median ratio of the estimate to the measured $K_{\mathrm{sym}}$, by contrast, stays between $7.5$ and $7.6$ for the step version and between $72$ and $74$ for the trajectory version at every $N$: most of the gap to the measured KL is approximation error and does not shrink with $N$.

\subsection{Performance prediction: protocol and full results}
\label{app:predict}
This section gives the details of \S\ref{sec:predicting} and the results of all predictors. Models, reference trajectories, and scoring are described in Appendix~\ref{app:base100}.

\paragraph{Tasks and average success rates.}
A success rate is the fraction of tasks solved; run conditions are listed in Table~\ref{tab:setups}. No values are missing across the 4 benchmarks, 50 tasks each, and 100 models. The tasks are 50 from ALFWorld valid\_unseen, 25 each from TextWorld treasure hunter easy and medium, 25 each from BFCL simple\_python and multi\_turn, and 25 each from MiniWoB++ navigate-tree and form-sequence. The four-benchmark average is the success rate over all 200 tasks, and the average over the three other environments is the success rate over the 150 tasks of TextWorld, BFCL, and MiniWoB++.

\paragraph{Predictors and intervals.}
We hold out one model at a time and predict it from the remaining 99. kNN uses Euclidean distance in centered coordinates, and the number of neighbors $k$ is chosen from $\{1,3,5,10\}$ by inner leave-one-out on the 99 training models. Ridge uses a linear kernel, with the penalty $\alpha$ chosen from five values, $\{10^{-3},10^{-2},10^{-1},1,10\}$ times the trace scale of the centered kernel. Size plus coordinates is a Ridge regression with $\log_{10}|\theta_i|$ as an unpenalized covariate. All hyperparameters are chosen by inner leave-one-out on the 99 training models, and predictions are clipped to $[0,1]$. For differences from the size baseline, 95\% intervals are the 2.5th and 97.5th percentiles of 2{,}000 bootstrap resamples of models. We compute them for both the difference in Pearson $r$, shown in bold in the main-text table, and the difference in RMSE, shown by arrows in the appendix table.
\begin{figure*}[!htbp]
  \centering
  \includegraphics[width=\textwidth]{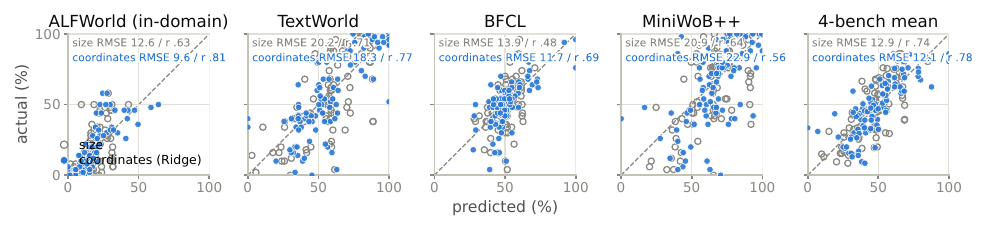}
  \caption{Predicting success rates without running the agents, with leave-one-model-out over 100 models. The horizontal axis is the prediction and the vertical axis the measured success rate. Open gray circles show the size baseline, a regression on $\log|\theta_i|$; blue points show coordinates alone, a Ridge regression on centered coordinates (the Ridge column of Table~\ref{tab:pred}). Panels from left: ALFWorld, TextWorld, BFCL, MiniWoB++, and the four-benchmark average.}
  \label{fig:pred}
\end{figure*}
\begin{table*}[!htbp]
  \centering\footnotesize
  \caption{All predictors for performance prediction, with leave-one-model-out over 100 models. Each cell gives RMSE in points / Pearson $r$. $\downarrow$ and $\uparrow$ mark a 95\% bootstrap interval of the RMSE difference from the size baseline that lies below or above zero. The training mean is a constant prediction, so no $r$ is shown.}
  \label{tab:pred_full}
  \setlength{\tabcolsep}{4pt}
  \renewcommand{\arraystretch}{1.1}
  \begin{tabular}{@{}l cccccc@{}}
    \toprule
    Predictor & ALFWorld & TextWorld & BFCL & MiniWoB++ & 4-bench avg. & 3 other envs. \\
    \midrule
    \multicolumn{7}{@{}l}{\textbf{Information available without running}} \\
    \quad Training mean & 16.4 / ---$^{\uparrow}$ & 29.0 / ---$^{\uparrow}$ & 16.1 / ---$^{\uparrow}$ & 27.5 / ---$^{\uparrow}$ & 19.2 / ---$^{\uparrow}$ & 21.4 / ---$^{\uparrow}$ \\
    \quad Size $\log|\theta_i|$ & 12.6 / 0.63 & 20.2 / 0.71 & 13.9 / 0.48 & 20.9 / 0.64 & 12.9 / 0.74 & 14.6 / 0.72 \\
    \quad Family mean & 14.7 / 0.46 & 26.8 / 0.44$^{\uparrow}$ & 15.5 / 0.34 & 27.6 / 0.24$^{\uparrow}$ & 17.7 / 0.42$^{\uparrow}$ & 19.9 / 0.40$^{\uparrow}$ \\
    \quad Mean log-likelihood $\bar\ell$ & 16.0 / 0.20$^{\uparrow}$ & 27.2 / 0.33$^{\uparrow}$ & 15.1 / 0.32 & 26.6 / 0.22$^{\uparrow}$ & 18.0 / 0.34$^{\uparrow}$ & 19.9 / 0.34$^{\uparrow}$ \\
    \quad Size + $\bar\ell$ & 12.4 / 0.65 & 19.2 / 0.74 & 13.4 / 0.54 & 20.7 / 0.65 & 12.2 / 0.77 & 13.8 / 0.76 \\
    \addlinespace[3pt]
    \multicolumn{7}{@{}l}{\textbf{Centered coordinates}} \\
    \quad kNN & 9.0 / 0.88$^{\downarrow}$ & 20.3 / 0.76 & 11.8 / 0.69$^{\downarrow}$ & 28.2 / 0.28$^{\uparrow}$ & 13.5 / 0.74 & 15.3 / 0.72 \\
    \quad Ridge & 9.6 / 0.81$^{\downarrow}$ & 18.3 / 0.77 & 11.7 / 0.69$^{\downarrow}$ & 22.9 / 0.56 & 12.1 / 0.78 & 14.0 / 0.75 \\
    \quad Size + Ridge & 9.4 / 0.82$^{\downarrow}$ & 18.0 / 0.78$^{\downarrow}$ & 12.1 / 0.65$^{\downarrow}$ & 21.1 / 0.63 & 10.9 / 0.83$^{\downarrow}$ & 12.5 / 0.81$^{\downarrow}$ \\
    \quad Size + $\bar\ell$ + Ridge & 9.2 / 0.83$^{\downarrow}$ & 17.1 / 0.80$^{\downarrow}$ & 11.0 / 0.72$^{\downarrow}$ & 20.3 / 0.67 & 10.3 / 0.84$^{\downarrow}$ & 12.3 / 0.81$^{\downarrow}$ \\
    \addlinespace[3pt]
    \multicolumn{7}{@{}l}{\textbf{Uncentered coordinates}} \\
    \quad kNN & 9.2 / 0.83$^{\downarrow}$ & 19.5 / 0.73 & 12.5 / 0.62 & 23.2 / 0.52 & 12.8 / 0.77 & 14.8 / 0.72 \\
    \quad Ridge & 9.0 / 0.83$^{\downarrow}$ & 19.4 / 0.74 & 11.0 / 0.72$^{\downarrow}$ & 22.8 / 0.55 & 12.1 / 0.78 & 14.2 / 0.75 \\
    \quad Size + Ridge & 8.9 / 0.84$^{\downarrow}$ & 16.5 / 0.82$^{\downarrow}$ & 10.9 / 0.73$^{\downarrow}$ & 20.2 / 0.67 & 10.3 / 0.85$^{\downarrow}$ & 11.6 / 0.84$^{\downarrow}$ \\
    \bottomrule
  \end{tabular}
\end{table*}

\paragraph{Results of all predictors.}
Table~\ref{tab:pred_full} and Figure~\ref{fig:pred} cover the training mean, size, family, $\bar\ell$, their combinations, and predictors on centered and uncentered coordinates. Two trends stand out. Adding $\bar\ell$ to size barely changes the predictions. Predictions from uncentered coordinates are about as good as those from centered coordinates. Using Spearman rank correlation instead of Pearson correlation does not change the ordering of the methods.

\subsection{Behavior: data, metrics, and additional results}
\label{app:behaviour}
This section defines the action sequences, distances, neighbors, and tests of \S\ref{sec:behaviour}, and reports measures, controls, and robustness checks not shown in the main-text table.

\subsubsection{Definitions}
\paragraph{Action sequences.}
We use one trajectory per model and task; run conditions are listed in Table~\ref{tab:setups}. For the 50 ALFWorld valid\_unseen tasks, the trajectories come from the same runs as the performance labels. An action sequence is the sequence of commands for the 50 TextWorld treasure-hunter tasks, the sequence of function names extracted from model outputs for the 25 BFCL multi\_turn tasks, and the sequence of action strings for the 50 MiniWoB++ tasks. The action type is the leading verb, function name, or tool name. An action is judged invalid by the environment's set of admissible actions in ALFWorld, by the environment's feedback that it could not parse the command in TextWorld, by a turn without a function call in BFCL, and by a tool error in MiniWoB++.

\paragraph{Behavioral distances.}
The four aspects in the main text are the base-2 Jensen--Shannon distance between action-type frequency distributions~\citep{endres2003metric}, the Levenshtein distance divided by the longer length~\citep{levenshtein1966,rabanser2026reliability}, the Jaccard distance between sets of action 2-grams, and the disagreement rate of the first action from the same initial state~\citep{bouchard2026trajectory,gonuguntla2026replay}. The first three are computed per task and averaged over the tasks, at least 10, recorded for both models. The appendix tables add four more measures. The divergence position is $1-k/\min(T_i,T_j)$, where $k$ is the number of steps before the first differing action. The conditional disagreement counts disagreement at step $t$ only for pairs whose action sequences agree up to step $t-1$, summed over $t=2,\dots,5$; only pairs with at least 10 such cases are used. For the first action, conditional disagreement, and divergence position, two invalid actions are not counted as agreeing. The profile distance is the Euclidean distance between per-model vectors standardized within each benchmark; each vector contains task averages of the repetition rate $\frac{T-|\text{distinct actions}|}{T-1}$~\citep{gioacchini2024agentquest}, the consecutive repetition rate, exploration breadth measured as distinct actions$/T$, the invalid-action rate, which is the complement of grounding accuracy~\citep{ma2024agentboard}, the share of information-gathering actions, the format-error rate, steps$/$limit, the rate of hitting the step limit, and the early-termination rate. For BFCL we also include the rate of empty turns. The success pattern is the fraction of tasks on which the two models' outcomes differ, $\frac1{|\mathcal T|}\sum_t\abs{s_{it}-s_{jt}}$.

\paragraph{Choosing neighbors and testing.}
Each model's five neighbors are the closest models in centered coordinates by Euclidean distance, in size by the difference in $\log|\theta|$, in $\bar\ell$ by the difference in mean log-likelihood on the reference trajectories, in family by taking models of the same family in order of size, and in success rate by the difference in measured success rate on that benchmark. For families with fewer than five other members, the remaining neighbors are taken from other families in order of size. The effect size is the mean behavioral distance to the remaining models minus the mean distance to the five neighbors, averaged over models and divided by the standard deviation of the behavioral distance over all pairs. Figure~\ref{fig:map}(c) applies the same definition to the \mbox{$k$-th} nearest model alone; ranks $k\le5$ are shown individually and larger ranks as bin averages. Significance is assessed by a two-sided Wilcoxon signed-rank test~\citep{wilcoxon1945} paired by model. The comparison between coordinate neighbors and success-rate neighbors, marked $\dagger$ in the main-text table, uses the same paired test: it asks whether the distance to success-rate neighbors minus the distance to coordinate neighbors, divided by the all-pairs standard deviation, is positive. These are paired tests with models as the unit; pairs share models and no correction for multiple comparisons is applied, so the significance markers are descriptive guides rather than exact error rates. The main conclusions rest on effects that are large (at least $0.5$ for coordinate neighbors in ALFWorld) and consistent across measures rather than on individual markers; small differences marked as significant, such as those between coordinate and success-rate neighbors in the other environments, should be read as suggestive.

\subsubsection{Robustness of the neighbor effect}
The main conclusion is that coordinate neighbors behave alike when run. We check that this does not depend on the choice of behavioral distance, on how performance is controlled, or on how the coordinates are built.

\paragraph{Does it hold for other behavioral distances?}
Part A of Table~\ref{tab:behaviour_full} adds the divergence position, conditional disagreement, profile, and success pattern to the four distances of the main text. The divergence position and conditional disagreement, which measure divergence early in the episode, follow the same pattern as the first action. For the profile, which aggregates repetition and invalid-action rates, and for the success pattern, neighbors chosen by success rate are closer than coordinate neighbors. These measures reflect how often an agent fails and are therefore tied closely to its success rate.

\paragraph{Does it survive regressing out performance?}
Instead of comparing with neighbors chosen by success rate, part B of Table~\ref{tab:behaviour_full} removes the effect of performance by regression. We regress each behavioral distance on the difference in success rate, the pair's mean success rate and its square, and the difference in $\log|\theta|$, and compute the effect of coordinate neighbors on the residuals. We control for the mean success rate as well because strong models converge on the shortest solutions and therefore act alike. For the success pattern, we instead subtract its expectation under a null distribution generated by the curveball algorithm~\citep{strona2014curveball}, which shuffles outcomes while keeping each model's success rate and each task's difficulty fixed. After this control, the effect of coordinate neighbors remains on every aspect in ALFWorld, the environment of the coordinates, and mainly on the local and early-step measures in the other environments.
\begin{table*}[!htbp]
  \centering\footnotesize
  \caption{Neighbor effect sizes for all behavioral distances, including those not in the main text; definitions as in Table~\ref{tab:behaviour}, 100 models. A: raw behavioral distances, comparing how neighbors are chosen; bold marks the largest value among the four neighbor types available without running. B: effect of coordinate neighbors on residuals after regressing out success rate and size. Unmarked values have $p<0.01$, $^{+}$ marks $p<0.05$, and gray values are not significant.}
  \label{tab:behaviour_full}
  \setlength{\tabcolsep}{4pt}
  \renewcommand{\arraystretch}{1.1}
  \begin{tabular}{@{}l ccccccccc@{}}
    \toprule
    & \multicolumn{2}{c}{Composition} & \multicolumn{1}{c}{Global} & \multicolumn{1}{c}{Local} & \multicolumn{3}{c}{Early steps} & \multicolumn{2}{c}{Outcome} \\
    \cmidrule(lr){2-3} \cmidrule(lr){4-4} \cmidrule(lr){5-5} \cmidrule(lr){6-8} \cmidrule(lr){9-10}
    & Type JS & Profile & Edit & 2-gram & First & Cond. & Diverg. & Success & Pattern \\
    \midrule
    \multicolumn{10}{@{}l}{\textbf{A}\quad Choice of neighbors} \\[1pt]
    \multicolumn{10}{@{}l}{\emph{ALFWorld}} \\
    \quad Coordinates & \textbf{0.52} & \textbf{0.63} & \textbf{0.88} & \textbf{0.96} & \textbf{0.61} & \textbf{0.51} & \textbf{0.70} & \textbf{0.60} & 0.32 \\
    \quad Size & 0.23 & 0.40 & 0.40 & 0.41 & 0.30 & 0.26 & 0.34 & 0.41 & 0.27 \\
    \quad Family & 0.28 & 0.48 & 0.62 & 0.68 & 0.49 & 0.43 & 0.62 & 0.49 & \textbf{0.35} \\
    \quad $\bar\ell$ & \textcolor{gray}{0.05} & 0.23 & 0.18$^{+}$ & 0.20 & 0.11 & 0.15 & 0.16$^{+}$ & 0.20 & 0.17 \\
    \quad Reference: success rate & 0.30 & 0.78 & 0.65 & 0.74 & 0.41 & 0.31 & 0.52 & --- & 0.61 \\
    \addlinespace[2pt]
    \multicolumn{10}{@{}l}{\emph{TextWorld}} \\
    \quad Coordinates & \textbf{0.33} & \textbf{0.52} & \textbf{0.52} & \textbf{0.54} & 0.21 & \textbf{0.51} & \textbf{0.43} & \textbf{0.61} & \textbf{0.65} \\
    \quad Size & \textcolor{gray}{0.10} & 0.35 & 0.25 & 0.27$^{+}$ & 0.14 & 0.32 & \textcolor{gray}{0.17} & 0.58 & 0.56 \\
    \quad Family & 0.27 & 0.47 & 0.42 & 0.44 & \textbf{0.25} & 0.46 & 0.34 & 0.56 & 0.61 \\
    \quad $\bar\ell$ & \textcolor{gray}{0.06} & 0.20 & \textcolor{gray}{0.14} & \textcolor{gray}{0.14} & \textcolor{gray}{0.04} & \textcolor{gray}{0.16} & \textcolor{gray}{0.10} & 0.27 & 0.29 \\
    \quad Reference: success rate & 0.22 & 0.80 & 0.50 & 0.52 & 0.16 & 0.47 & 0.31$^{+}$ & --- & 1.06 \\
    \addlinespace[2pt]
    \multicolumn{10}{@{}l}{\emph{BFCL}} \\
    \quad Coordinates & 0.18 & \textbf{0.36} & 0.26 & \textbf{0.38} & \textbf{0.32} & 0.22 & \textbf{0.46} & \textbf{0.47} & 0.28 \\
    \quad Size & 0.10$^{+}$ & 0.20 & 0.13 & 0.21 & 0.19 & 0.21 & 0.30 & 0.27 & 0.25 \\
    \quad Family & \textbf{0.21} & 0.29 & \textbf{0.26} & 0.32 & 0.25 & \textbf{0.28} & 0.40 & 0.30 & \textbf{0.29} \\
    \quad $\bar\ell$ & \textcolor{gray}{0.03} & 0.08 & \textcolor{gray}{0.06} & 0.09$^{+}$ & 0.09 & \textcolor{gray}{0.04} & 0.11$^{+}$ & 0.15 & 0.17 \\
    \quad Reference: success rate & 0.23 & 0.49 & 0.30 & 0.39 & 0.21 & 0.26 & 0.50 & --- & 0.66 \\
    \addlinespace[2pt]
    \multicolumn{10}{@{}l}{\emph{MiniWoB++}} \\
    \quad Coordinates & \textbf{0.24} & \textbf{0.34} & \textbf{0.38} & \textbf{0.50} & \textbf{0.50} & \textbf{0.11$^{+}$} & \textbf{0.48} & 0.39 & 0.40 \\
    \quad Size & 0.12 & 0.16 & 0.20 & 0.27 & 0.35 & \textcolor{gray}{$-$0.02} & 0.30 & \textbf{0.42} & \textbf{0.44} \\
    \quad Family & 0.22 & 0.23 & 0.32 & 0.47 & 0.35 & \textcolor{gray}{0.08} & 0.36 & 0.37 & 0.37 \\
    \quad $\bar\ell$ & 0.08 & 0.10 & 0.10$^{+}$ & 0.16$^{+}$ & 0.09$^{+}$ & 0.10$^{+}$ & 0.11$^{+}$ & 0.20 & 0.21 \\
    \quad Reference: success rate & 0.13 & 0.63 & 0.37 & 0.42 & 0.28 & 0.19$^{+}$ & 0.29 & --- & 0.98 \\
    \midrule
    \multicolumn{10}{@{}l}{\textbf{B}\quad Coordinate neighbors, residuals after removing success rate and size} \\[1pt]
    ALFWorld & 0.22 & 0.09$^{+}$ & 0.61 & 0.68 & 0.31 & 0.20 & 0.34 & --- & 0.24 \\
    TextWorld & 0.04$^{+}$ & 0.03$^{+}$ & 0.18 & 0.23 & $-$0.13 & 0.15 & 0.14 & --- & 0.22 \\
    BFCL & \textcolor{gray}{0.04} & 0.09 & 0.09 & 0.20 & 0.13 & 0.10$^{+}$ & 0.24 & --- & $-$0.15$^{+}$ \\
    MiniWoB++ & \textcolor{gray}{0.03} & 0.11 & \textcolor{gray}{0.09} & 0.31 & 0.31 & \textcolor{gray}{0.01} & 0.28 & --- & 0.20 \\
    \bottomrule
  \end{tabular}
\end{table*}

\paragraph{Does it depend on how the coordinates are built or on families?}
The top half of Table~\ref{tab:behaviour_raw} chooses neighbors in uncentered coordinates. The effect is weaker than with centered coordinates, presumably because the overall likelihood level then influences which models count as neighbors. The bottom half excludes models of the same family from the coordinate neighbors. The effect remains on all four benchmarks, so coordinate neighbors do not behave alike merely because they belong to the same family.
\begin{table*}[!htbp]
  \centering\footnotesize
  \caption{Variants of coordinate neighbors, with definitions as in Table~\ref{tab:behaviour_full}. Top: neighbors in uncentered coordinates. Bottom: coordinate neighbors restricted to models from other families. Neighbors in centered coordinates are in Table~\ref{tab:behaviour_full}. Unmarked values have $p<0.01$, $^{+}$ marks $p<0.05$, and gray values are not significant.}
  \label{tab:behaviour_raw}
  \label{tab:behaviour_xfam}
  \setlength{\tabcolsep}{4pt}
  \renewcommand{\arraystretch}{1.08}
  \begin{tabular}{@{}l ccccccccc@{}}
    \toprule
    & \multicolumn{2}{c}{Composition} & \multicolumn{1}{c}{Global} & \multicolumn{1}{c}{Local} & \multicolumn{3}{c}{Early steps} & \multicolumn{2}{c}{Outcome} \\
    \cmidrule(lr){2-3} \cmidrule(lr){4-4} \cmidrule(lr){5-5} \cmidrule(lr){6-8} \cmidrule(lr){9-10}
    & Type JS & Profile & Edit & 2-gram & First & Cond. & Diverg. & Success & Pattern \\
    \midrule
    \multicolumn{10}{@{}l}{\textbf{Neighbors in uncentered coordinates}} \\[1pt]
    \multicolumn{10}{@{}l}{\quad Raw behavioral distance} \\
    \qquad ALFWorld & 0.32 & 0.51 & 0.65 & 0.73 & 0.45 & 0.38 & 0.52 & 0.64 & 0.45 \\
    \qquad TextWorld & 0.19 & 0.49 & 0.37 & 0.40 & 0.12$^{+}$ & 0.39 & 0.27$^{+}$ & 0.66 & 0.68 \\
    \qquad BFCL & 0.12 & 0.31 & 0.21 & 0.31 & 0.22 & 0.22 & 0.45 & 0.47 & 0.35 \\
    \qquad MiniWoB++ & 0.14 & 0.28 & 0.25 & 0.40 & 0.38 & 0.12$^{+}$ & 0.34 & 0.38 & 0.38 \\
    \multicolumn{10}{@{}l}{\quad Residuals after removing success rate and size} \\
    \qquad ALFWorld & 0.10 & \textcolor{gray}{$-$0.03} & 0.47 & 0.54 & 0.21 & 0.14 & 0.28 & --- & 0.29 \\
    \qquad TextWorld & 0.06$^{+}$ & 0.11 & 0.15 & 0.19 & $-$0.06$^{+}$ & 0.14 & 0.14 & --- & 0.20 \\
    \qquad BFCL & \textcolor{gray}{0.05} & 0.07 & 0.10 & 0.20 & 0.07$^{+}$ & 0.16 & 0.28 & --- & \textcolor{gray}{0.04} \\
    \qquad MiniWoB++ & \textcolor{gray}{0.03} & 0.16 & 0.11$^{+}$ & 0.31 & 0.28 & \textcolor{gray}{0.07} & 0.26 & --- & 0.16 \\
    \midrule
    \multicolumn{10}{@{}l}{\textbf{Coordinate neighbors restricted to other families}} \\[1pt]
    \multicolumn{10}{@{}l}{\quad Raw behavioral distance} \\
    \qquad ALFWorld & 0.47 & 0.56 & 0.60 & 0.61 & 0.45 & 0.36 & 0.41 & 0.54 & 0.26 \\
    \qquad TextWorld & 0.31 & 0.44 & 0.45 & 0.44 & 0.13 & 0.43 & 0.34 & 0.52 & 0.54 \\
    \qquad BFCL & \textcolor{gray}{0.08} & 0.24 & 0.15$^{+}$ & 0.24 & 0.23 & \textcolor{gray}{0.06} & 0.26 & 0.40 & 0.19 \\
    \qquad MiniWoB++ & 0.21 & 0.31 & 0.26 & 0.28 & 0.42 & \textcolor{gray}{0.02} & 0.36 & 0.35 & 0.30 \\
    \multicolumn{10}{@{}l}{\quad Residuals after removing success rate and size} \\
    \qquad ALFWorld & 0.21 & 0.08$^{+}$ & 0.31 & 0.31 & 0.16 & 0.13 & \textcolor{gray}{0.09} & --- & \textcolor{gray}{0.04} \\
    \qquad TextWorld & 0.07 & \textcolor{gray}{0.01} & 0.17 & 0.17 & $-$0.18 & 0.12 & \textcolor{gray}{0.06} & --- & 0.18 \\
    \qquad BFCL & \textcolor{gray}{$-$0.01} & \textcolor{gray}{0.02} & \textcolor{gray}{0.04} & 0.12 & 0.08 & \textcolor{gray}{$-$0.01} & \textcolor{gray}{0.10} & --- & $-$0.28 \\
    \qquad MiniWoB++ & \textcolor{gray}{0.02} & 0.08$^{+}$ & \textcolor{gray}{$-$0.06} & \textcolor{gray}{0.08} & 0.24 & \textcolor{gray}{$-$0.05} & 0.17 & --- & \textcolor{gray}{0.05} \\
    \bottomrule
  \end{tabular}
\end{table*}

\subsubsection{Reliability of the behavioral distances}

\paragraph{Are the distances reproducible across tasks?}
Table~\ref{tab:reliability} shows the correlation between distance matrices computed on two random halves of the tasks. The composition, global, and local measures reach $0.85$ or higher, so different task sets give nearly the same distances. The early-step measures, such as the first action, and the success pattern are somewhat lower, especially on BFCL, which has only 25 tasks.
\begin{table}[!htbp]
  \centering\footnotesize
  \caption{Reliability of the behavioral distances: correlation between distance matrices computed on two random halves of the tasks, averaged over 30 splits; all standard deviations are at most 0.08. The profile and conditional disagreement cannot be split by task and are omitted.}
  \label{tab:reliability}
  \setlength{\tabcolsep}{1.5pt}
  \begin{tabular}{@{}l cccc@{}}
    \toprule
    & ALFWorld & TextWorld & BFCL & MiniWoB++ \\
    \midrule
    \multicolumn{5}{@{}l}{\textbf{Reliability}\quad correlation across task halves} \\
    Type JS & 0.97 & 0.89 & 0.89 & 0.95 \\
    Edit distance & 0.91 & 0.92 & 0.89 & 0.92 \\
    2-gram & 0.88 & 0.90 & 0.87 & 0.85 \\
    First action & 0.78 & 0.82 & 0.70 & 0.91 \\
    Divergence & 0.76 & 0.90 & 0.62 & 0.91 \\
    Success pattern & 0.77 & 0.80 & 0.68 & 0.84 \\
    \bottomrule
  \end{tabular}
\end{table}

\section{Composing harnesses: details}
\label{app:harness}
This appendix supplements \S\ref{sec:harness}. Appendix~\ref{app:harnesssetup} gives the setup and Appendix~\ref{app:additivity} the full additivity table and figure.

\subsection{Harness implementations, models, and references}
\label{app:harnesssetup}
\begin{table}[!htbp]
  \centering\scriptsize
  \caption{Setup of the multi-harness experiment: 12 LLMs $\times$ 12 harnesses, scored on the benchmarks that have reference trajectories (ALFWorld and SWE-Bench use 6 of the 12 LLMs). $N$ is the number of reference trajectories and ``steps'' the number of reference steps scored by all agents.}
  \label{tab:hsetup}
  \setlength{\tabcolsep}{2.5pt}
  \begin{tabular}{@{}l l l l@{}}
    \toprule
    & Reference trajectories (source) & $N$ / steps & LLMs scored \\
    \midrule
    $\tau$-bench & HAL airline & 359 / 4{,}213 & 12 \\
    ALFWorld & own, 3 LLMs $\times$ 12 harnesses & 210 / 2{,}299 & 6 \\
    SWE-Bench & HAL SWE-agent & 100 / 3{,}916 & 6 \\
    \bottomrule
  \end{tabular}
\end{table}
\paragraph{Input construction.}The input is assembled in the order fixed part, history selected by context management, plan and reasoning for the current step, and action instruction. The fixed part consists of the system instructions, tool definitions, and initial observation (including the task instruction) and is shared by all harnesses. One block of the external history is an executed action and the resulting observation; past plans and reasoning are not included. The budget for the history part is 8{,}192 tokens, and the whole input is at most 16{,}384 tokens including the reserve for the output; when the budget is exceeded, the oldest blocks are dropped first and, if necessary, the latest observation is truncated to its head and tail.

\paragraph{Components.}
\begin{itemize}[leftmargin=*,itemsep=1pt,topsep=1pt]
  \item \textbf{Context management}: windowed passes the initial observation and the last four blocks verbatim. retrieve additionally selects, from the older history, blocks relevant to the current observation by BM25~\citep{robertson2009bm25} (weight 0.7) and recency (0.3) and inserts up to 1{,}536 tokens of them. summarize splits the older history, summarizes the parts, and merges them hierarchically into a JSON memory of at most 1{,}024 tokens that replaces the history. If the history has at most four blocks and uses at most 75\% of the budget, neither retrieval nor summarization is applied and the history is passed raw.
  \item \textbf{Planning}: replan rewrites from scratch, at every step and in at most 256 tokens, the confirmed progress, the next subgoal, and the observation that would confirm its achievement, and adds them to the input of the current step.
  \item \textbf{Acting}: think generates, for the next single action, step-by-step reasoning of at most 512 tokens that organizes the goal, constraints, assumptions, and expected effect of the action, and adds it to the input of the current step (the plan, if any, is also included in the input).
\end{itemize}
Plans and reasoning are not carried over to the next step.

\paragraph{Generation and scoring.}Generation (rollouts) and scoring (teacher forcing) use the same construction; for scoring, each harness builds its input from the same reference history $h_t$, and the raw log-likelihood of the fixed reference action is computed. The action being scored, future observations, and rewards are not used in the input.
\paragraph{LLMs.}The 12 LLMs are Qwen3-8B~\citep{yang2025qwen3}, Qwen3.8-27B~\citep{qwen2026qwen38}, Qwen3.5-27B/9B/2B~\citep{qwen2026qwen35}, Gemma4-E4B/12B/31B~\citep{gemmateam2026gemma4}, GPT-OSS-20B/120B~\citep{openai2025gptoss}, GLM-4.7-Flash~\citep{zai2026glm47flash}, and Devstral-Small-2-24B~\citep{rastogi2025devstral}, all among the 100 models of \S\ref{sec:predicting}. SWE-Bench was scored with 6 of them (Qwen3-8B, Qwen3.8-27B, Gemma4-E4B, GPT-OSS-20B/120B, GLM-4.7-Flash).
\paragraph{Reference trajectories and coordinates.}$\tau$-bench: 359 HAL~\citep{kapoor2025hal} airline traces mapped to the 30 tasks of the train split of $\tau^2$-bench~\citep{barres2025tau2}; of their 4{,}236 steps, the 4{,}213 scored by all 144 agents are used for the coordinates. SWE-Bench: 100 HAL traces of SWE-agent~\citep{yang2024sweagent}, 3{,}916 steps. ALFWorld: 210 train rollouts of our own 3 LLMs $\times$ 12 harnesses, of which the 2{,}299 steps scored by all 72 agents (6 LLMs $\times$ 12 harnesses) are used. All coordinates are the centered step version of \S\ref{sec:vectors}.
\paragraph{Change in the mean log-likelihood.}The component removed from the coordinates by centering, $\bar\ell(m,A)-\bar\ell(m,\base)$, is $+9.6$ for retrieve, $+3.5$ for summarize, $-4.2$ for replan, and $-6.2$ for think on $\tau$-bench, in nats per step (medians over LLMs), and its sign is shared by at least 11 of the 12 LLMs. Context management makes the reference actions more natural, whereas planning and reasoning move the agent away from them.

\begin{table*}[!t]
  \centering\scriptsize
  \caption{Full additivity table (medians over LLMs: cosine, norm ratio, relative reconstruction error). $\tau$-bench uses 12 LLMs; ALFWorld and SWE-Bench use 6 of the 12 LLMs (ALFWorld uses only the 2{,}299 reference steps scored by all 72 agents). The pooled rows take the median over all LLM--harness cases: the five two-component harnesses (60 cases on $\tau$-bench, 30 on ALFWorld and SWE-Bench) and the two three-component harnesses (24 / 12 / 12 cases).}
  \label{tab:additivity_full}
  \setlength{\tabcolsep}{4pt}
  \begin{tabular}{@{}l ccc ccc ccc@{}}
    \toprule
    & \multicolumn{3}{c}{$\tau$-bench} & \multicolumn{3}{c}{ALFWorld} & \multicolumn{3}{c}{SWE-Bench} \\
    \cmidrule(lr){2-4}\cmidrule(lr){5-7}\cmidrule(lr){8-10}
    & cos & Norm ratio & Rel. error & cos & Norm ratio & Rel. error & cos & Norm ratio & Rel. error \\
    \midrule
    retrieve + replan & 0.97 & 0.99 & 0.25 & 0.62 & 0.95 & 0.94 & 0.97 & 0.97 & 0.25 \\
    retrieve + think & 0.94 & 1.00 & 0.34 & 0.61 & 0.95 & 0.95 & 0.96 & 0.96 & 0.28 \\
    summarize + replan & 0.91 & 0.96 & 0.45 & 0.60 & 0.89 & 1.00 & 0.78 & 0.89 & 0.72 \\
    summarize + think & 0.87 & 0.98 & 0.51 & 0.63 & 0.91 & 0.93 & 0.77 & 0.86 & 0.76 \\
    replan + think & 0.81 & 0.74 & 0.77 & 0.70 & 0.73 & 1.03 & 0.74 & 0.73 & 0.95 \\
    \midrule
    Two components, pooled & 0.90 & 0.98 & 0.46 & 0.63 & 0.89 & 0.97 & 0.81 & 0.90 & 0.65 \\
    Three components, pooled & 0.82 & 0.87 & 0.68 & 0.52 & 0.72 & 1.25 & 0.77 & 0.76 & 0.84 \\
    \bottomrule
  \end{tabular}
\end{table*}

\begin{figure*}[!t]
  \centering
  \includegraphics[width=\textwidth]{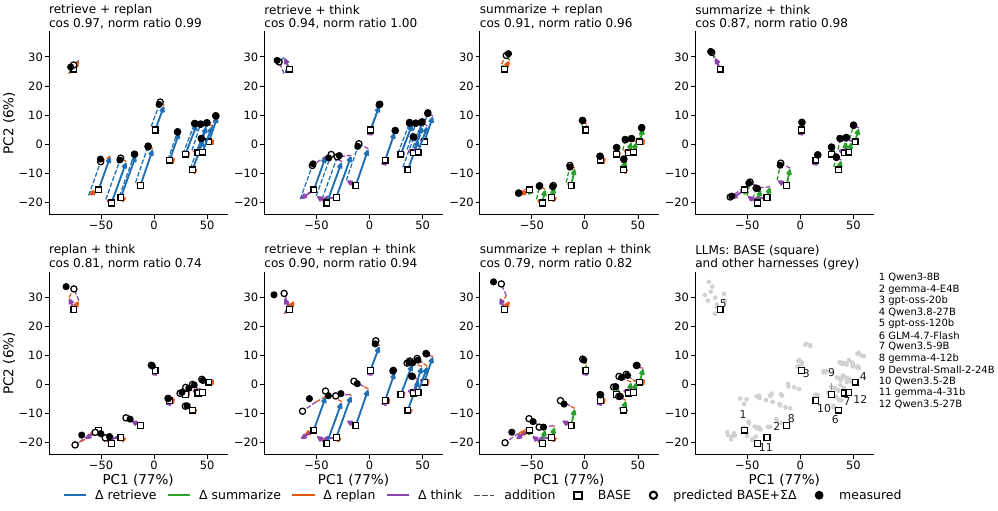}
  \caption{Additivity for all seven compositions on $\tau$-bench (12 LLMs), in the PC1--PC2 projection of Figure~\ref{fig:additivity}; axes are scaled independently. The last panel identifies the LLMs by their BASE positions (gray: the other harnesses). Two-component compositions are drawn as parallelograms, three-component ones as chains $\mathrm{BASE}+\Delta_A+\Delta_B+\Delta_C$. Titles give the cosine and norm ratio in the full coordinate space (medians over LLMs).}
  \label{fig:additivity_all}
\end{figure*}
\subsection{Additivity: full tables}
\label{app:additivity}
Table~\ref{tab:additivity_full} and Figure~\ref{fig:additivity_all} give the full results. On ALFWorld and SWE-Bench (6 LLMs) the direction of the composed displacement also agrees with the prediction, but more weakly than on $\tau$-bench (median cosines over all LLMs $\times$ 5 two-component compositions of $0.63$ and $0.81$), and on ALFWorld the relative reconstruction error reaches about 1. Each column is a median of per-LLM values, so the columns do not satisfy the identity below exactly. The ALFWorld reference trajectories are rollouts of the agents under study themselves, and only 2{,}299 of their steps are scored by all 72 agents.
\paragraph{Identity between norm ratio and cosine.}With $v=z(A{+}B)-z(\base)$, $\hat v=\Delta_A+\Delta_B$, $e=\hat v-v$, norm ratio $r=\norm{v}/\norm{\hat v}$, and $c=\cos(\hat v,v)$, the relative reconstruction error is $\norm e^2/\norm v^2=1+r^{-2}-2c/r$. Even when the norms agree, an error remains whenever $c<1$.

\end{document}